\documentclass[12pt,journal,onecolumn]{IEEEtran}

\usepackage{diagbox}
\usepackage{enumitem}
\usepackage[utf8]{inputenc} 
\usepackage[T1]{fontenc}    
\usepackage{hyperref}       
\usepackage{url}            
\usepackage{booktabs}       
\usepackage{amsfonts}       
\usepackage{nicefrac}       
\usepackage{microtype}      
\usepackage{amsmath}
\usepackage{amssymb}
\usepackage{graphicx}
\usepackage{caption,subcaption}
\usepackage[ruled,norelsize]{algorithm2e}
\usepackage{amsthm}
\newtheorem{theorem}{Theorem}
\newtheorem{lemma}{Lemma}
\newtheorem{corollary}{Corollary}
\newtheorem{remark}{Remark}

\ifCLASSOPTIONcompsoc
  \usepackage[nocompress]{cite}
\else
  \usepackage{cite}
\fi

\makeatletter
\newcommand{\removelatexerror}{\let\@latex@error\@gobble}
\makeatother

\begin{document}

\title{How to Query An Oracle?\\ Efficient Strategies to Label Data}
\author{Farshad~Lahouti,
        Victoria~Kostina,
        and~Babak~Hassibi
\IEEEcompsocitemizethanks{\IEEEcompsocthanksitem The authors are with the Electrical Engineering Department, California Institute of Technology, Pasadena,
CA, 91125. 
E-mail: flahouti@ieee.org,vkostina@caltech.edu,hassibi@caltech.edu}
\thanks{Accepted for publication in IEEE Transactions on Pattern Analysis and Machine Intelligence  \copyright 2021 IEEE. Personal use of this material is permitted. Permission from IEEE must be obtained for all other uses, in any current or future media, including 
reprinting/republishing this material for advertising or promotional purposes, creating new 
collective works, for resale or redistribution to servers or lists, or reuse of any copyrighted 
component of this work in other works.}
}

%
%

\markboth{IEEE Transactions on Pattern Analysis and Machine Intelligence,~Vol.~XX, No.~X, 2021}%
{Lahouti \MakeLowercase{\textit{et al.}}: How to Query An Oracle? Efficient Strategies to Label Data}
%



\IEEEtitleabstractindextext{%
\begin{abstract}
We consider the basic problem of querying an expert oracle for labeling a dataset in machine learning. This is typically an expensive and time consuming process and therefore, we seek ways to do so efficiently. The conventional approach involves comparing each sample with (the representative of) each class to find a match. In a setting with $N$ equally likely classes, this involves $N/2$ pairwise comparisons (queries per sample) on average. We consider a $k$-ary query scheme with $k\ge 2$ samples in a query that identifies (dis)similar items in the set while effectively exploiting the associated transitive relations. We present a randomized batch algorithm that operates on a round-by-round basis to label the samples and achieves a query rate of $O(\frac{N}{k^2})$. In addition, we present an adaptive greedy query scheme, which achieves an average rate of $\approx 0.2N$ queries per sample with triplet queries. For the proposed algorithms, we investigate the query rate performance analytically and with simulations. Empirical studies suggest that each triplet query takes an expert at most 50\% more time compared with a pairwise query, indicating the effectiveness of the proposed $k$-ary query schemes. We generalize the analyses to nonuniform class distributions when possible.
\end{abstract}

\begin{IEEEkeywords}
Machine learning, labeling datasets, clustering, classification, entity resolution
\end{IEEEkeywords}}

\maketitle

\IEEEdisplaynontitleabstractindextext

%
\IEEEpeerreviewmaketitle

\ifCLASSOPTIONcompsoc
\IEEEraisesectionheading{\section{Introduction}\label{sec:introduction}}
\else
\section{Introduction}
\label{sec:introduction}
\fi
\IEEEPARstart{C}{lustering} and classification are the basic elements of many machine learning systems and applications. A labeled and well curated dataset is essential to design a high performing classifier. Clustering algorithms also benefit from labeled data. Labeling of data is done by experts or non-expert crowds. In the latter case, it may contain some errors, and in either case, collecting the labels and creating reliable datasets is an expensive and time consuming process. In this work, we investigate ways to label a data set efficiently with as few expert queries as possible.

As a motivating example, we consider the important and widely researched entity resolution (ER) problem. Entity resolution is the problem of identifying in a dataset the multiple instances that are related to a given entity or an object. For example, we may seek all pictures that are related to a given point of interest in a database of images; or we may wish to find all articles authored by a researcher. The problem is visibly complicated, if we take into account the variety of angels and light conditions of photography in the first case, or in the second case, people with the same names, and ways people write their names and different affiliations. A survey of ER and related research is reported in \cite{4016511}.

Researchers have proposed a variety of interesting heuristic solutions and a few with theoretical guarantees to address the ER problem. The anatomy of an ER solution may be described as follows. (1) A similarity measure is assumed, based on which a probability is computed for any pair of samples to be related, i.e., to be from the same cluster. (2) The samples are ordered for pairwise queries based on the said probability and possibly using a metric quantifying the expected gain \cite{Whang13}. Some pairs may also be immediately decided as either a match or a mismatch if the probability is very high or very low. In \cite{Verroios15} a metric for ordering based on a maximum likelihood criterion is presented. (3) The queries are posed to experts or to a crowd of non-experts and labels are collected, possibly within a constrained budget. (4) A post-processing of the results then produces the final clusters. This last stage could for example be based on clustering, or majority voting when crowdsourcing is used. Positive or negative transitive relations can be used in stage 2 to reduce the number of queries or in stage 4 to handle possible label errors \cite{Wang13}\cite{Vesdapunt14}. For example, if previous queries showed that samples $s_1$ and $s_2$ match and that $s_2$ and $s_3$ match, then via positive transitivity there is no need to query $s_2$ and $s_3$ as they obviously match. If previous queries further revealed that $s_1$ and $s_4$ do not match, then via negative transitivity it immediately follows that $s_4$ does not match $s_2$ and $s_3$ either. For more efficient labeling, \cite{Wang12} and \cite{Marcus11} have explored the notion of batching pairwise queries in so-called human intelligence tasks.
		
ER schemes may be evaluated from multiple perspectives \cite{Chen18}\cite{Kopke10}: (1) Effectiveness or performance of clustering (for example in terms of recall and precision); (2) Efficiency, or the number of queries required per sample to achieve this performance; (3) Operation and scalability, i.e., whether the scheme is adaptive or non-adaptive, whether it runs online or in batch, and whether it is parallelizable; and, (4) Genericity, or how and whether the scheme may be applied to different scenarios. For example, the Jaccard similarity function is popular when only dealing with textual data. More generally, how well a given similarity function represents the true probability that two samples match is quite sensitive to the context. As a consequence, most prior work on labeling is domain-specific.

A query scheme dubbed \emph{$k$-ary incidence coding}, or $k$IC, is presented in \cite{Lahouti16} and its effectiveness is analyzed in boosting the reliability of crowdsourcing. A $k$IC query is essentially a comparison of $k$ samples with the transitive relations embedded in the query design. They also investigated the information theoretic limits of label crowdsourcing in certain settings. \cite{Aria17} studied clustering problem in presence of a similarity function (side information) from an information-theoretic perspective and proposed algorithms and lower-bounds on its performance.

In this work, we focus on efficient labeling of datasets in clustering and classification applications. We assume we have access to an oracle (an expert), who provides correct answers to each query at a given cost. We seek efficient query schemes that are domain-agnostic; we thus assume that no similarity measure is available. We consider $k$IC queries and examine the advantage of querying $k\ge 2$ objects at once for labeling a dataset into $N\ge 2$ classes. Empirical studies suggest that each triplet query takes an expert at most 50\% more time compared with a pairwise query. This is while it provides three times more comparisons, indicating the potential of larger query sizes \cite{Ramya88}. We emphasize that the problems of interest typically have a large number of classes, $N$, and taking $k$'s that are small or moderate is more practical. Three algorithms are presented: (1) A basic algorithm that runs on a sample-by-sample basis; (2) A batch processing randomized algorithm that operates in multiple rounds and achieves a query rate of $O(\frac{N}{k^2})$ with $k$IC; and (3) An adaptive greedy algorithm that achieves the best performance at $\approx 0.2N$ queries per sample for $3$IC. Each of the algorithms are assessed both numerically and analytically. To the best of our knowledge the batch processing algorithm is the first query scheme reported in the literature with proven query rate of better than $O(N/k)$.

This article is organized as follows. We begin with the problem setup and background in Section \ref{PrelSec}. The three proposed algorithms are described and analyzed in Sections \ref{Algo1Sec}, \ref{Algo2Sec} and \ref{Algo3Sec}. Some numerical results, performance analyses and experiment results are reported in Section \ref{PerfSec}. The article is concluded in Section \ref{ConcSec}. We present the proofs to theorems and corollaries with some additional details and examples in
the appendices.

\section{Preliminaries}\label{PrelSec}

\subsection{Overview}
\begin{figure*}[t!]
\centering
  \vspace*{\fill}
  \centering
  \includegraphics[width=0.95\textwidth]{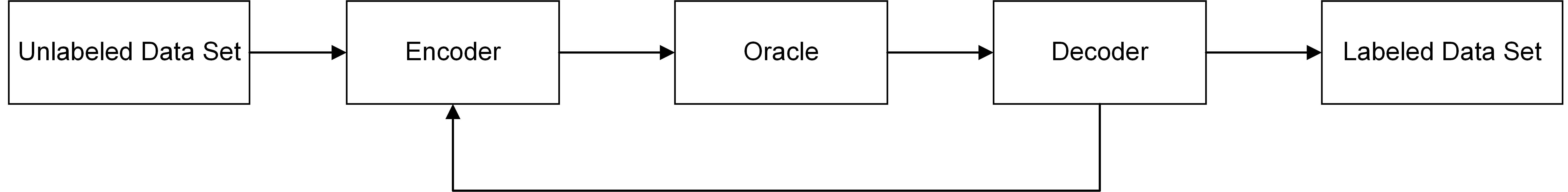}
\caption{System Model}
\label{SystemFig}
\end{figure*}

\begin{figure*}[t!]
\centering
  \vspace*{\fill}
  \centering
  \includegraphics[width=0.8\textwidth]{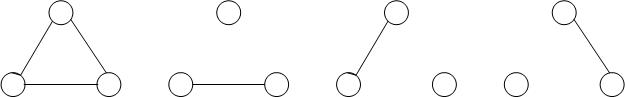}
  \par\vfill
  \vspace{0.25cm}
  \includegraphics[width=0.8\textwidth]{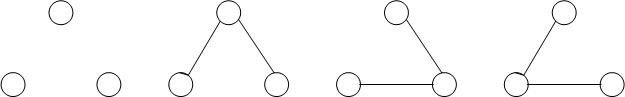}
\caption{A $3$IC query. Circles represent the samples to be labeled. An edge shows that two samples are from the same class. For $N=2$, the first row shows the possible valid responses and the second row the invalid responses. For $N>2$, the first element of the second row is also a valid response.}
\label{3ICFig}
\end{figure*}

We consider labeling of a dataset of $L\gg 1$ samples, where the number of classes $N$ may be unknown. Figure \ref{SystemFig} shows the system model. We assume a general $k$IC query scheme, composed of an encoder and a decoder. The $k$IC encoder assembles $k \ge 2$ samples from the dataset in a single query to learn about their (dis)similarities from the oracle. The decoder then collects the responses and attempts to classify the samples. The oracle is assumed to be able to identify similar and dissimilar samples in a $k$IC query without error. The case with errors may be handled in conjunction with crowdsourcing schemes as discussed for future research in Section \ref{ConcSec}. Since we here focus on application agnostic schemes for labeling a dataset, we do not assume a similarity function at the encoder. However, we consider a feedback from the decoder to the encoder. The feedback can be used for the encoder to learn the number of classes and estimate their sizes as it processes the dataset. The feedback also allows the encoder to strategize for an optimized assembly of samples in a subsequent query based on the current set of results received from the oracle. 

In the sequel, we elaborate on the dataset models, the $k$IC query scheme and measures of assessing its efficiency. Although many of the results we present are for general values of $k$, we specialize them to particular values, for example $2$ or $3$, when applicable. In general, one would not use a large $k$ when dealing with human experts for query. Figure~\ref{3ICFig} shows a $3$IC scheme. Due to the transitive relations, out of all eight possibilities for the query response, only four or five (depending on the value of $N$) may be valid. 

A conventional query, when we directly ask about the label of a given sample, is equivalent to a $(N+1)$IC query, composed of the sample and the $N$ representatives of each of the classes. It is directly evident that for large $N$, such a query could be a challenging one that is potentially prone to error; hence, motivating $k$IC with limited $k$. We offer a few example use cases here: (1) Large number of classes: In labeling pictures that are related to different points of interest in a database of relevant images, there are plenty of such places many of them may be less popular and less known. As such directly identifying the corresponding class for an image is challenging, however, comparisons are in general easier. This is typical of many ER problems, which involve detecting duplicates in databases with large number of classes (entries) \cite{4016511}. (2) Subjective issues: Consider the problem of labeling images for whether or how much they may be suitable for children. In this case, comparison type queries are in general easier and more reliable. (3) Nonuniform datasets: When dealing with highly nonuniform datasets, which contain classes of noticeably different sizes, the conventional approach does not leave any room for innovation, as there is always a single query per sample required. As we shall demonstrate in the sequel, using $k$IC, we can design more sophisticated labeling schemes. As an example, if we use a $3$IC, and opt to query two samples with the biggest class at first, it is highly likely that the two samples are labeled right there. This shows a query rate improvement of up to $50\%$ compared to the conventional scheme. (4) Expert knowledge: Consider the problem of labeling a dataset of dog images with their breeds, especially when the number of breeds are large. With $k$IC, the problem breaks down to comparison of a few dog images at once. For a non-expert comparisons could be easier as opposed to directly identifying the bread. In all these examples, $k$IC queries provide room for more innovation.

\subsection{Data Set: Source}\label{SSource}
Consider a dataset $\mathcal{X}=\{X_1,\hdots, X_L\}$ composed of $L$ items, e.g., dog images. 
In practice, there is certain function $B(X)\in \mathcal{B}(\mathcal{X})$ of the items (e.g., the breed of the dogs) that is of interest and is here considered as the source. The value of this function is to be determined by the oracle for the given dataset. In the case of clustering, $B(X_i)= B_j\in\mathcal{B}(\mathcal{X})=\{B_1, \hdots, B_N\}$ indicates the bin or the class to which the item $X_i$ ideally belongs.
We have $B(X_1, \hdots, X_n)=B(X^n)=(B(X_1),\hdots, B(X_n))$. Without loss of generality, we consider $\mathcal{B}(\mathcal{X})=\{1, \hdots, N\}$ in the sequel. The number of clusters, $|\mathcal{B}(\mathcal{X})|=N$, may or may not be known a priori. In developing the proposed query schemes, we deal with an arbitrary sample of this dataset, denoted by $s \in \mathcal{X}$. When two samples match, it means that they belong to the same class.

The classes may be of different sizes and as such we consider a probability distribution for the classes as 
	\begin{equation}
	    \boldsymbol{\pi}=\{\pi_1, \hdots, \pi_N\}.
	    \label{GeneralPiEq}
	\end{equation}
In case, the class distributions are uniform, we have
	\begin{equation}
	    \boldsymbol{\pi}^u=\{1/N, \hdots, 1/N\}.
	\end{equation}
A probability distribution of interest for classes in machine learning datasets is the Zipf distribution. It is a discrete power law distribution that approximates many types of data in natural languages, physical sciences and social sciences and is described as follows
	\begin{equation}
	    \pi^z(i; \nu, N)=\frac{1}{\eta\, i^\nu}, \hspace{1cm} 1\le i \le N,
	\end{equation}
where $\nu$ is an exponent parameter describing the shape of the distribution and $\eta=\sum_{j=1}^{N}\frac{1}{j^\nu}$ is a normalization factor.
In parts of our analyses, we specialize the results to a particular class of nonuniform class distributions. In this setting there is one class of distinct size, but the rest have the same sizes.
	\begin{equation}
	    \boldsymbol{\pi}^n(x)=\bigg\{\frac{\alpha'}{N}, \hdots, \frac{\alpha'}{N},\frac{\alpha}{N}\bigg\}
	\label{NonuniformClassEq}
	\end{equation}
	where $\alpha'=\frac{N-\alpha}{N-1}$ and $x := \frac{\alpha}{(N-1)\alpha'}$ is a parameter that determines the size (probability) of the distinct cluster with respect to others. It is straightforward to see that $\alpha=\frac{xN}{1+x}$ and $\alpha'=\frac{N}{(N-1)(1+x)}$. For $0<\epsilon \ll 1$, having $x=\epsilon$ and $\epsilon N \ll 1$ would indicate one small cluster and $x=1/\epsilon$ indicates one large cluster. In the former case, $x=\epsilon$, we have $\pi_i=\frac{1}{(N-1)(1+\epsilon)}, 1\le i \le N-1$ and $\pi_N=\frac{\epsilon}{1+\epsilon}$, i.e., 
		\begin{equation}
	    \boldsymbol{\pi}^n(x=\epsilon) = \bigg\{\frac{1}{(N-1)(1+\epsilon)}, \hdots, \frac{1}{(N-1)(1+\epsilon)},\frac{\epsilon}{1+\epsilon}\bigg\}.
	\label{NonuniformClassEq1}
	\end{equation}
In the latter case, $x=1/\epsilon$, we have $\pi_i=\frac{\epsilon}{(N-1)(1+\epsilon)}\approx \frac{\epsilon}{N-1}, 1\le i \le N-1$ and $\pi_N=\frac{1}{1+\epsilon} \approx 1-\epsilon$, i.e., 
		\begin{equation}
	    \boldsymbol{\pi}^n(x=1/\epsilon) = \bigg\{\frac{\epsilon}{N-1}, \hdots, \frac{\epsilon}{N-1},1-\epsilon\bigg\}.
	\label{NonuniformClassEq2}
	\end{equation}
This is an abstraction of the Zipf distribution, and Pareto-type distributions, which lends itself well to analysis and provides interesting insights as we shall elaborate in the sequel.

\subsection{Query Efficiency}
When two samples $s_1$ and $s_2$ meet in a query (with possibly other samples participating) we denote this event by $(s_1,s_2)\in {\cal Q}$. A sample is \emph{settled} in a query when there is a match (or merger); this is denoted by $s\in {\cal L}$, where ${\cal L}$ is the labeled set. If this occurs for a sample following $q$ queries we denote it by $s\in {\cal L}_q$, where $\mathcal{L}_q$ indicates the set of samples that are settled in $q$th query. $E_{q}$ describes the event that a sample is settled in $q$th query, i.e.,  
\begin{equation}
E_q:=
\begin{cases}
1 & s \in \mathcal{L}_q, s \notin \mathcal{L}_{q-1}, s \notin \mathcal{L}_{q-2}, \hdots\\
0 & \text{otherwise}
\end{cases}
\nonumber
\end{equation}
This indicates that the sample has already participated in $q-1$ queries without success, but is successful in the subsequent query.

The efficiency of a query scheme may be quantified in terms of the number of queries made for labeling of an arbitrary sample. This is in general a random variable and hence may be characterized by a probability distribution, herein referred to as the \emph{query distribution}, $P(Q=q):=\mathbb{E}_{(s,\theta)}[\Pr\{E_q=1\}$]. Here, $\theta$ is a latent variable representing possible randomness in the query algorithm. The algorithms we consider in this work are domain agnostic and do not use any similarity function, and hence treat different samples, $s$, the same, as objects. Still, the class distribution does reflect in the query distribution. Depending on the application and the query scenario, we can use different statistics of the query distribution as measures of efficiency. In general, a useful metric quantifying the efficiency of a query scheme is the average number of queries per sample, or the \emph{query rate}: 
\begin{equation}
R=\frac{1}{\kappa}\mathbb{E}[Q] \hspace{1cm} \text{queries/sample},
\end{equation} 
where $\kappa$ indicates the number of samples that participate in a query for labeling in a given query scheme (with $k$IC, $1\le \kappa\le k$). Note that in $k$IC with $\kappa\le k$, $k - \kappa$ samples in the query is known a priori (see for example Sections \ref{Algo1Sec} and \ref{Algo3Sec}). For a given query scheme, we can attempt to statistically analyze these efficiency metrics. Alternatively, we can empirically compute these metrics based on a sufficiently large dataset. If it takes $W$ queries with a specific algorithm to label a large dataset of size $L$, then $\frac{W}{\kappa L}$, also approximates the query rate. 

In labeling a dataset of size $L$, if the price we pay for each query is fixed at $\eta'$, then the overall budget we require for labeling is $\eta=R\times \eta'$ in which $R$ is the average query rate of the labeling scheme in use. While the focus of this research is on design of query efficient labeling schemes (minimizing $R$), this indicates the role of query pricing. A reduced query rate allows room for higher query prices with a given budget, hence providing a wealth of design trade-offs to optimize labeling budget. This issue is also reflected in pricing of crowdsourcing schemes, as is for example discussed in the Appendix of [10] for $k$IC queries. Pricing of a query in general can be designed from cost or value perspectives, taking into account practical (or market) constraints. The cost is influenced by factors such as the time or effort it takes to respond to a query, and the engagement model and the level of expertise of the people involved. The value on the other hand depends on how helpful or critical the label is in the specific application. Different commercial solutions in this space have already adopted a wide variety of business models and pricing schemes.

\subsection{Query Scheme: $k$-ary Incidence Coding ($k$IC)}\label{SCoding}

In a $k$IC query, the purpose is to identify those samples among a set of $k$ samples that are similar. The queries in fact are posed as inquiring the elements of a binary incidence matrix, ${\bf A}$, whose rows and columns correspond to $X$ \cite{Lahouti16}. In this case, ${\bf A}(X_1,X_2 )=1$ indicates that the two are members of the same cluster and ${\bf A}(X_1,X_2 )=0$ indicates otherwise. The matrix is symmetric and its diagonal is $1$. We refer to this query scheme as Binary Incidence Coding. If we present three samples at once and ask the user to classify them (put them in similar or distinct bins);  it is as if we ask about three elements of the same matrix, i.e., ${\bf A}(X_1,X_2), {\bf A}(X_1,X_3)$ and ${\bf A}(X_2,X_3)$ (Ternary Incidence Coding). In general, if we present $k$ samples as a single query, it is equivalent to inquiring about ${k \choose 2}$ (choose $2$ out of $k$ elements) entries of the matrix ($k$-ary Incidence Coding or $k$IC). As we elaborate below, out of the $2^{{k \choose 2}}$ possibilities, a number of the choices remain invalid and the redundancy that is thus captured makes more efficient querying. To assess this, we quantify the number of possible responses to a $k$IC query, when we deal with clustering of a dataset into $N$ classes. In this setting, a $k$IC query may have samples from one up to $\min(k,N)$ classes. Assuming the query has exactly $i$ classes, the following lemma quantifies the number of possible valid responses $g(k,i)$.   
\begin{lemma}\label{TheLemma}
With $k$ vertices and $i$ disjoint subgraph cliques, the number of possible graph realizations is given by

\begin{equation}
g(k,i)=\sum_{j=\lceil{\frac{k}{i}}\rceil}^{k-i+1} h(k,i,j)
\end{equation}
where $i\le k$, $h(k,i,j)$ is given in \eqref{hEQ},
\begin{figure*}[h]
\begin{equation}\label{hEQ}
h(k,i,j)=
\begin{cases}
{k \choose j} \times \sum_{l=\lceil{\frac{k-j}{i-1}}\rceil}^{\min(j,k-j-i+2)} \frac{1}{1+\rho^*(k,i-1,j,l)}\times h(k-j,i-1,l)& \lceil{\frac{k}{i}}\rceil \leq j \leq k-i+1\\
0 & \text{otherwise}
\end{cases}
\end{equation}
\end{figure*}
$h(l,l,1)=h(l,1,l)=1, g(l,1)=g(l,l)=1, \forall l\ge 1$ and $h(k,i,1)=0, \forall k \neq i$, and
\begin{equation}
\rho^*(k,i,j,l)=
\begin{cases}
     \lfloor{\min\big(\frac{k-j-i}{j-1},\frac{k-j}{j}\big)}\rfloor & j=l\\
     0 & otherwise.
\end{cases}
\end{equation}
\end{lemma}
\begin{proof}
The proof and more details are provided in the appendix.
\end{proof}
In \eqref{hEQ}, $h(k,i,j)$ represents the number of possibilities in $k$IC queries with exactly $i$ cliques (classes) and $j$ items in its largest class. In some occasions, we have more than one bin with the largest number of items, as elaborated in the proof, the factor $\frac{1}{1+\rho^*}$ in \eqref{hEQ} is meant to count for that.
The number of valid responses in a $k$IC query where the samples are from $N$ possible classes, is then given by 
\begin{equation}\label{fFunDef}
f(k,N)=\sum_{i=1}^{\min{(N,k)}} g(k,i).    
\end{equation}
With $k$ samples in a query, the redundancy that is captured by $k$IC due to the transitive relations reduces the number of possible choices from $2^{k \choose 2}$ to $f(k,N)$. We may quantify this as ${k \choose 2}-\log_2 f(k,N)$ in bits. For $N=5, k=2,\hdots,10$, the redundancy is $\{0\quad   0.68\quad    2.09\quad    4.30\quad    7.34\quad   11.26\quad   16.09\quad   21.86\quad   28.63\}$ bits. This is an indication and quantification of how exploiting transitive relations in $k$IC makes it an efficient query scheme.

\section{Basic Query Scheme with $k$IC}\label{Algo1Sec}
Consider the clustering of a large dataset into $N$ classes. For every sample, we can query the oracle for comparison of the sample with one or more class representatives ($k$IC, $k\ge 2$) to identify a match. We refer to this scheme as the \emph{basic query scheme}. In case the classes (and their number) are not known in advance, we can create a new class whenever we exhaust all the existing classes and do not find a match. This scheme is elaborated in Algorithm~\ref{Algo1}.  
Since Algorithm~\ref{Algo1} queries each sample with the classes in sequence, its efficiency is sensitive to the ordering of the classes if they have different probabilities. For an optimal query performance with an oracle that responds without error to pairwise comparisons ($2$IC), it is optimal to arrange the classes in reducing order of their likeliness \cite{massey94}. 

In labeling of a large dataset, Algorithm~\ref{Algo1} first runs through a transient phase until all classes and the order of their sizes (probabilities) are identified. At this point, it reaches a steady state. Theorem~\ref{Algo1Theorem} presents the query rate of Algorithm~\ref{Algo1} or the average number of $k$IC queries until a sample is labeled, in steady state. For the uniform class distribution and pairwise comparisons ($2$IC), similar results are presented in \cite{Davidson14}.

\begin{figure}[!t]
 \removelatexerror
  \begin{algorithm}[H]
\caption{Basic Query Scheme with $k$IC and Unknown Number and Size of Classes}
\label{Algo1}
   \begin{enumerate}
    \item 
    Initialize:
    \begin{enumerate}
    \item
Take $k$ samples at once and compare.
    \item
Consider the distinct subsets as classes and consider one sample in each subset as the representative.
    \end{enumerate}
    \item \label{Algo1Step2}
Arrange the class representatives in reducing order of their class sizes
\item
Take a new sample from the dataset:
    \begin{enumerate}
        \item 
    If $\geq k - 1$ classes have already been created, compare the sample with $k-1$ existing classes at a time ($k$IC) in sequence. Merge if there is a match.
    \item
    If only $0<k'<k-1$ classes are available, compare with only that many classes ($(k'+1)$IC).
    \item
If all existing classes are exhausted and there is no match, create a new class.
    \end{enumerate}
    \item
    Continue from step~\ref{Algo1Step2} until all samples are labeled.
    \end{enumerate}
\end{algorithm}
\end{figure}

\begin{theorem}\label{Algo1Theorem}
The query rate of Algorithm 1 with $k$IC and class distribution $\boldsymbol{\pi}$ in the steady state is given by
\begin{equation}\label{Algo1TheoremEq}
R_1^{kIC}=\sum_{i=1}^{\frac{N}{k-1}}i \pi'_{i} \hspace{1cm} \text{queries/sample}
\nonumber
\end{equation}
in which $\pi'_i=\sum_{j=(i-1)(k-1)+1}^{i(k-1)}\pi_j, 0<i\leq N/(k-1)$ and we assume $k-1$ divides $N$. Furthermore, as prescribed in step 2 of this algorithm, it is optimal to arrange the class representatives in reducing order of their class sizes.
\end{theorem}
\begin{proof}
The proof of this theorem and the resulting two corollaries are provided in the appendix.
\end{proof}
The next corollary specializes the results to the case where the class distribution is uniform.
\begin{corollary}
The query rate of Algorithm 1 with $k$IC and uniform class distribution in the steady state is given by
\begin{equation}
R_1^{kIC}(\boldsymbol{\pi}^u)=\frac{N+k-1}{2(k-1)} \simeq \frac{N}{2(k-1)}\hspace{1cm} \text{queries/sample},
\nonumber
\end{equation}
where the approximation holds for $N\gg k>1$. 
\label{Algo1Cor1}
\end{corollary}
As evident in Corollary~\ref{Algo1Cor1}, with $k$IC and for large $N$, the query rate of Algorithm~\ref{Algo1} improves linearly with $k$. 

In the next corollary, we consider the non-uniform class distributions presented in \eqref{NonuniformClassEq1} and \eqref{NonuniformClassEq2} with $k$IC, $2\leq k <N$, and present bounds on the query rate. The bounds are obtained considering the two extreme cases, when the classes are arranged in reducing (as prescribed in step 2 of Algorithm~\ref{Algo1}) or increasing order of their sizes. We use the results for insights into the transient performance of Algorithm~\ref{Algo1} in the sequel.
\begin{corollary}
The query rate of Algorithm 1 with $k$IC, large $N$ and nonuniform class distributions $\boldsymbol{\pi}^n(\epsilon)$ and $\boldsymbol{\pi}^n\left(1/\epsilon \right)$, $\epsilon\ll 1$, is given by
\begin{equation}\label{Algo1Cor2Eq1}
\frac{N(1+\epsilon)+k-3}{2(k-1)}\lesssim R_1^{kIC}\left(\boldsymbol{\pi}^n(\epsilon)\right)\lesssim \frac{N+k-1-N\epsilon}{2(k-1)} 
\end{equation}
\begin{equation}\label{Algo1Cor2Eq2}
1+\frac{N\epsilon}{2(k-1)}\lesssim R_1^{kIC}\left(\boldsymbol{\pi}^n\left(1/\epsilon\right)\right) \lesssim \frac{N(2-\epsilon)}{2(k-1)}. 
\end{equation}
where $2\leq k < N$ and $\lesssim$ means $\leq$ up to $O(N\epsilon)$.
\label{Algo1Cor2}
\end{corollary}

In using Algorithm~\ref{Algo1}, in case we do not know the class distribution in advance, we simply estimate it as we make progress through labeling a dataset. Accordingly, we may only arrange the classes according to the reducing order of their sizes as prescribed in step 2 of the algorithm with the estimate at hand. Corollary~\ref{Algo1Cor2} shows how the knowledge (or lack thereof) of the class distribution can affect the query rate by possibly imperfect arrangement of the classes. Specifically, it shows that the Algorithm~\ref{Algo1} with $k$IC and $\boldsymbol{\pi}^n(\epsilon)$ is robust to handling of a cluster of smaller size, $\epsilon$: It operates at the same query rate order, $O(N)$, disregarding the ordering of the classes. 
An important observation here is how a class of large size impacts the query rate, $R_1^{kIC}$. As evident with $\boldsymbol{\pi}^n(\frac{1}{\epsilon})$, proper ordering of classes can change the query rate order from $O(N)$ to $O(1)$ in this setting. This indicates that in using Algorithm~\ref{Algo1}, in case we do not know the class distribution in advance, the suggested (periodic) ordering of the classes as we make progress could noticeably reduce the query rate in general. 

\section{Query Scheme with $k$IC and Batch Processing}\label{Algo2Sec}
In some machine learning applications, batch processing of queries are preferred. In such settings, all the queries are submitted as tasks at once and the queries do not vary on the fly. In this Section, we present an algorithm that lends itself to batch processing. The scheme is presented in Algorithm~\ref{Algo2} for the general case with $k$IC. The basic idea is that the samples are partitioned randomly to $k$ sample subsets and submitted as queries to the oracle. We merge those samples that match in a query and return one representative of every merged set to form a new batch. This process is repeated in rounds until no further mergers are possible. As evident the scheme easily accommodates the scenario with unknown number and distribution of classes. If the batch is of distribution $\boldsymbol{\pi}$, the next theorem shows how the distribution and the batch size evolve following processing by Algorithm~\ref{Algo2}, and presents the average query rate over $m\ge 1$ rounds, $R^{kIC}_{2, 1:m}$.

\begin{theorem}\label{Algo2Theorem}
When a batch of samples of size $L=L_1\gg 1$, with distribution $\boldsymbol{\pi}$ over $N$ classes, is processed by Algorithm~\ref{Algo2} in a single round by $k$IC, the resulting batch of samples will be of distribution $\boldsymbol{\pi}'$ over the classes, where
\begin{equation}
\pi'_i=\frac{1-(1-\pi_i)^k}{N-\sum_{j=1}^{N} (1-\pi_j)^k}, \hspace{1cm} 1\le i \le N,
\label{EqPiEvolution}
\end{equation}
and the probability that a sample is settled is 
\begin{equation}\label{EqSettled}
P(s \in {\cal L}) = 1-\frac{1}{k}\sum_{j=1}^{N} (1-\pi_j)^k.    
\end{equation} 
Over $m\ge 1$ rounds of Algorithm~\ref{Algo2}, when $L_m\gg 1$, with iterative use of \eqref{EqPiEvolution} and \eqref{EqSettled}, the batch size and the average query rate are given by   
\begin{equation}
\begin{array}{ccc}
L_m &=& L \Pi_{j=1}^{m-1}(1-P(s\in {\cal L}_{j}))\\\\
R^{kIC}_{2, 1:m}(\boldsymbol{\pi}) &=&\frac{1}{k}\frac{\sum_{j=1}^{m}\Pi_{i=1}^{j-1}(1-P(s\in {\cal L}_{i}))}{\sum_{j=1}^m P(s \in {\cal L}_j) \Pi_{i=1}^{j-1}\big(1-P(s\in {\cal L}_{i})\big)},
\end{array}
\end{equation}
where $P(s \in {\cal L}_r)$ is the probability that a sample is settled in round $r$.
\end{theorem}
\begin{proof}
The proof is provided in the appendix. 
\end{proof}

The next Corollary specializes the results to the case where the dataset is uniform, i.e., $\boldsymbol{\pi}=\boldsymbol{\pi}^u$.
\begin{corollary}\label{Algo2Cor1}
When a large batch of $L$ samples with uniform class distribution is processed with Algorithm~\ref{Algo2}, the class distribution remains unchanged in consecutive rounds, $1\le r \le m$, and we have
\begin{equation}
P(s\in {\cal L}_r) = 1-\frac{N}{k}\big(1-(1-\frac{1}{N})^k\big),
\end{equation}
when $L_m=L(1-P(s \in {\cal L}))^m \gg 1$. As such, the average query rate over round $r$ or until round $r$, $1\le r \le m$, is given by 
\begin{equation}
R^{kIC}_{2,r}(\boldsymbol{\pi}^u)=R^{kIC}_{2,1:r}(\boldsymbol{\pi}^u)=\frac{1}{k-N\big(1-(1-\frac{1}{N})^k\big)}\approx\frac{2N}{k(k-1)},
\end{equation}
where the approximation in the right hand side holds for $N\gg k>1$.
\end{corollary}
\begin{proof}
The proof is provided in the appendix.
\end{proof}

An interesting observation from the above Corollary is that when the class distribution is uniform and $N\gg 1$, then the average query rate improves quadratically with $k$. To our knowledge this is the first algorithm with provably better performance than $O(N/k)$ reported in the literature.

If $\boldsymbol{\pi}$ is non-uniform then chances of a few samples matching in a $k$IC query is higher for samples from bigger clusters. As such, starting from a non-uniform class distribution, in a single round of Algorithm~\ref{Algo2}, it is more likely that samples from bigger clusters, as opposed to those from smaller clusters, are settled. This implies that class probabilities of batches approach a uniform distribution through the processing rounds of Algorithm~\ref{Algo2}. The next Corollary specializes Theorem~\ref{Algo2Theorem} to the case where $\boldsymbol{\pi}=\boldsymbol{\pi}^n(\frac{1}{\epsilon})$, and demonstrates the above intuition in this case. In addition, since in Algorithm~\ref{Algo2}, the queries are formed in a random manner and we do not strategize in selection of the samples in a query, the next corollary shows how a big cluster affects the performance.

\begin{corollary}\label{Algo2Cor2}
When labeling a large batch with nonuniform class distribution $\boldsymbol{\pi}^n(\frac{1}{\epsilon}), \epsilon \ll 1$, using $k$IC and Algorithm~\ref{Algo2}, the class distribution evolves to 
\begin{equation}\label{Algo2Cor2Eq1}
\begin{array}{ccl}
\pi'_N&\approx& 1-k\epsilon\\
\pi'_i &\approx& \frac{k\epsilon}{N-1}; \quad i\in\{1,\hdots,N-1\} 
\end{array}
\end{equation}
over a single round and the average query rate is given by
\begin{equation}
R_2^{kIC}\big(\boldsymbol{\pi}^n\big(\frac{1}{\epsilon}\big)\big)\approx\frac{1}{k-1}+\epsilon'
\end{equation}
where $\epsilon'=\frac{k}{(k-1)^2}\epsilon$.
\end{corollary}
\begin{proof}
The proof is provided in the appendix.
\end{proof}
The above corollary shows that a class with much bigger size can drastically affect the performance and make its rate approximately independent of the number of classes, $N$. It is evident in \eqref{Algo2Cor2Eq1} how the size of the bigger cluster is reducing and those of the smaller clusters are increasing over a single round of Algorithm \ref{Algo2} towards a uniform class distribution. We will elaborate on this further in Section \ref{PerfSec}.

\begin{figure}[!t]
 \removelatexerror
  \begin{algorithm}[H]
\begin{enumerate}
\item Initialize the batch with the dataset to be labeled.
    \item Partition the batch into subsets of cardinality $k$.
    \item Query the oracle using the $k$IC on each subset.
    \item Merge the matched samples in each query. Return one representative of each collection of merged samples to the batch.
    \item If no merger is possible in step 4, stop; Otherwise, continue to step 2.
\end{enumerate}
\caption{$k$IC Query Scheme with Batch Processing}
\label{Algo2}
\end{algorithm}
\end{figure}

\section{Greedy Algorithm with $3$IC}\label{Algo3Sec}
\subsection{Algorithm Description} \label{Algo3Desc}
In this Section, we present an adaptive greedy query algorithm that, similar to the basic scheme in Algorithm~\ref{Algo1}, lends itself suitably to online applications. The proposed algorithm runs in a greedy fashion, i.e., the queries are formed with the purpose of maximizing the number of samples that may be labeled in any query. The algorithm operates in a round robin fashion, i.e., we index the classes (their representatives) in a specific order, and run $k$IC queries with $k-1$ samples from the dataset and one of the class representatives in sequence. In the sequel, we focus on a $3$IC realization of the proposed algorithm for a more clear exposition. The proposed greedy triplet query scheme is described in Algorithm~\ref{Algo3}.

\begin{figure*}[t]
\begin{minipage}[b]{\linewidth}
\centering
  \includegraphics[width=0.8\textwidth]{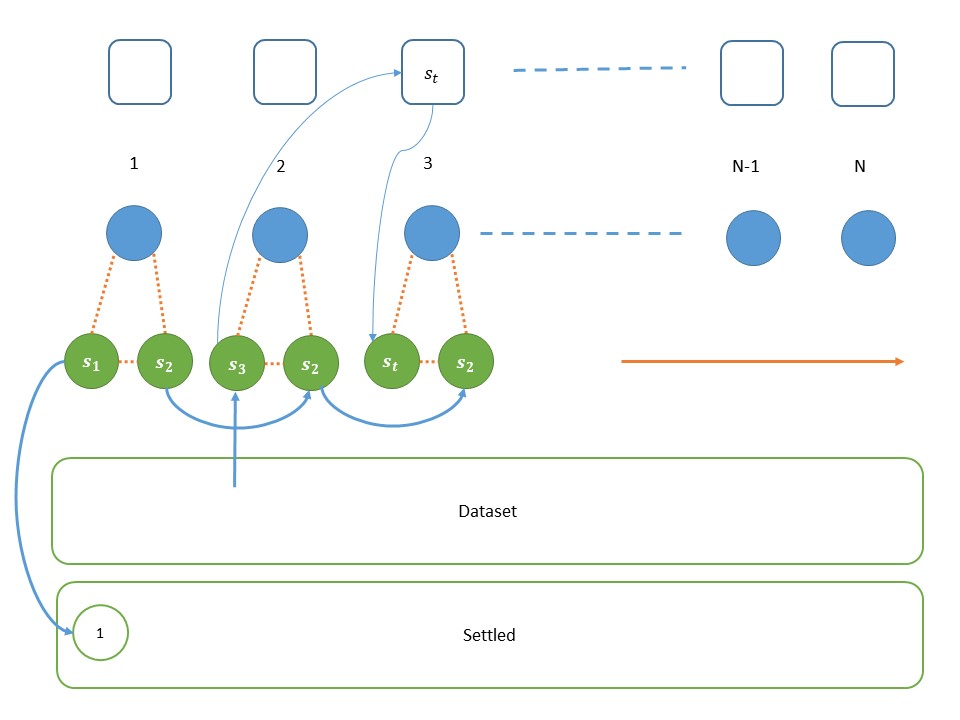}
\end{minipage}
\caption{A basic example of three queries with Algorithm~\ref{Algo3} and $3$IC. The blue circles show the class representatives. The squares with round corners show temporary bins. The green circles are samples to be labeled.}
\label{Algo3Figs}
\end{figure*}
Figure~\ref{Algo3Figs} shows an example of the first three queries in this Algorithm. In the first query, two samples, $s_1$ and $s_2$, are drawn from the dataset for a $3$IC query with the representative of the first bin (class). In this example, $s_1$ matches the bin and is thus labeled (settled). There was no match for $s_2$, so it advances to the next query with class $2$ and a new sample $s_3$ drawn from the dataset. In this example, there was no match in the second query. If we advance both $s_2$ and $s_3$ to the third query with the third class, the $(s_2$, $s_3)$ comparison would be redundant. As such, we store $s_3$ in a temporary bin associated with class $3$, $T_3$, for a future query with this class in the next round. Then $s_2$ advances to the third query with class $3$ and a sample $s_t$ which happens to be in $T_3$ from the previous round.   

The example shows key components of the Algorithm: (i) an ordered set of class representatives, (ii) a round robin query scheme with $3$IC and, (iii) a set of temporary bins one corresponding to each class. The class representatives are assumed given here, but as discussed in Section~\ref{Algo1Sec} they may be identified during the run time of the algorithm alternatively. The example also shows a few questions, answering which are key to design the algorithm: (i) If one or two samples are not settled in the current query, which one do we advance to the next query and (ii) which one(s) do we store in the temporary bin? (iii) When we draw new samples for a new query, when do we draw from the temporary bin and (iv) when do we draw from the dataset?

To motivate the design of the algorithm, we examine the probability that a given sample is settled in a query. Samples are queried with classes in sequence and in general a sample may join and start the query round with any of the $N$ class representatives. To take this into account, we define the {\it length} of a sample, $\ell(s)$, as the number of classes it has been queried with so far. Assume we compare two samples $s_1$ and $s_2$ with length $\ell(s_1)=i$ and $\ell(s_2)=j$ with a bin $s$ and $0 \leq j \leq i < N-1$. In this query, the total number of possibilities for the three element set of the samples and the bin representative is $(N-i)(N-j)$. For a uniformly distributed dataset, we can quantify the potential outcome of each query as follows:

\begin{itemize}
    \item 
	Case a (all three match): one instance; $p_a (i,j) = {1 \over {(N-i)(N-j)}}$
\item
	Case b ($s$ and $s_1$ match): $(N-j-1)$ instances; $p_b(i,j)={{N-j-1} \over {(N-i)(N-j)}}$
\item	
	Case c ($s$ and $s_2$ match): $(N-i-1)$ instances; $p_c (i,j)= {{N-i-1} \over {(N-i)(N-j)}}$
\item	
	Case d ($s_1$ and $s_2$ match): $(N-i-1)$ instances; $p_d(i,j)={{N-i-1} \over {(N-i)(N-j)}}$
\item	
	Case e (no match): $(N-i-1)(N-j-2)$ instances; $p_e(i,j)={{(N-i-1)(N-j-2)} \over {(N-i)(N-j)}}$ 
\end{itemize}

\begin{lemma}
In a single $3$IC query of Algorithm 3 involving two samples of length $i$ and $j$ and a given class representative over a uniform dataset, the number of samples that are successfully settled on average is given by:
\begin{equation}
SL={{2(N-i)+(N-j-1)}\over{(N-i)(N-j)}}.
\nonumber
\end{equation}
\end{lemma}
\begin{proof}
The proof is provided in the appendix.
\end{proof}
It is easy to show that $SL$ is a monotonically increasing function of $i$ and $j$, i.e., a sample with larger length is more likely to settle. In the context of Algorithm 3, a greedy approach to labeling which aims at maximizing the number of samples settled in a query, therefore prescribes the following query design strategy.
\begin{itemize}
    \item
    In cases b and c, it is advantageous to advance the current sample in query that is not yet settled to the next query, as opposed to drawing a new sample from the dataset. This sample is of a larger length and is hence more likely to settle next.
    \item
    In case d, when the two samples match, it is obviously advantageous to take the sample with larger length as the representative and advance it to the next query.
    \item
    In case e, where there is no match, it is better to advance the sample with larger length to the next query. We keep the other sample in the so-called temporary bin of the subsequent class to be queried in the next {\it round}.	There is a temporary bin corresponding to each of the classes, which can store one sample at a time: $\mathcal{T}=\{T_1, \hdots, T_N \}$. By next round, we mean the next time the subsequent class will be queried in the round robin scheme of Algorithm 3. This is to avoid forming redundant queries in an efficient manner.
    \item
    In forming a new query, we take the sample that is advanced from the previous query if available (in cases b, c, d, and e); otherwise we draw one new sample from the dataset. We also take another sample from the corresponding temporary bin (from the previous round) if available, otherwise, we take another new sample from the dataset to form the query.
\end{itemize}
The proposed triplet query scheme is described in Algorithm~\ref{Algo3}. In the sequel, we analyze the performance of this algorithm. The analysis is set up based on a Markov model for the algorithm, computing its transition probabilities and then the query rate of the algorithm.

\begin{figure}[h]
 \removelatexerror
  \begin{algorithm}[H]
\begin{enumerate}
    \item 
	Set Up:
	\begin{enumerate}
	    \item 
	There are $N$ classes with one representative sample in its corresponding bin.
	\item
	The classes are ordered and numbered in sequence in a preferred manner.
	\item
	There is a temporary bin corresponding to each of the classes, which can store one sample at a time: $\mathcal{T}=\{T_1, \hdots, T_N \}$. Initialize with $T_i=\emptyset$.
	\item
	During the query process, the classes are examined in sequence in a round robin fashion. Start from the first class $i=1$. 
	\item
	For each sample $s$ drawn from the dataset during this query process
	\begin{enumerate}
	    \item 
	    Consider (the index of) the class with which it started the query process as $start(s)$.
	    \item
	    Consider length of a sample, $\ell(s)$, as number of classes it has been queried with so far.
	\end{enumerate}
    \item
	Initialize the variable $sample2advance = \emptyset$. This variable stores the sample from the current query that is not settled and is intended to advance to the next query.
	\end{enumerate}
\item
Identify two samples for a triplet query with bin $i$
\begin{enumerate}
    \item
If $sample2advance \neq \emptyset$, take $s_1 \xleftarrow{} sample2advance$, otherwise draw a new sample as $s_1$. Set $sample2advance = \emptyset$
    \item
If $T_i \neq \emptyset$, $s_2 \xleftarrow{} T_i$, otherwise draw a new sample as $s_2$
\item
Update the length of the samples
\end{enumerate}
	\item
	Query $(s_1,s_2)$ with bin $i$.  
\begin{enumerate}
\item 
If both samples match with the bin, drop them in the bin. 
\item
If $s_1$ matches with the bin, but not $s_2$, place $s_1$ in the bin, $sample2advance \xleftarrow{}s_2$. 
\item
If $s_2$ matches with the bin, but not $s_1$, place $s_2$ in the bin, $sample2advance \xleftarrow{}s_1$. 
\item
If the samples match each other, but not with the bin. Merge the two samples. Set $sample2advance\xleftarrow{} \text{arg} \max (\ell (s_1), \ell(s_2))$
\item
If none of the samples matches, set $sample2advance\xleftarrow{}\text{arg} \max (\ell (s_1), \ell(s_2))$, place the other sample in temporary bin $i+1$. \end{enumerate}
\item
Continue from step 2 with the next bin (in a round robin fashion) until all samples are labeled. Note that $start(s)$ and $\ell(s)$ in 1.e together characterize the bins any sample has been queried with so far.
\end{enumerate}
\caption{Greedy Triplet Query Algorithm}
\label{Algo3}
\end{algorithm}
\end{figure}

\subsection{Performance Analysis}
In order to analyze the query rate of the Algorithm 3 in the steady state, we consider the following definitions. We define the event $E_{i,j}$ as the event that two samples of length $i$ and $j$ meet in a $3$IC query, i.e., 
\begin{equation}
E_{i,j}=
\begin{cases}
1 & \text{if} \quad \ell(s_1)=i,\ell(s_2)=j,(s_1,s_2) \in \mathcal{Q}\\
0 & \text{otherwise};
\end{cases}
\end{equation}
Note that the other sample in the $3$IC query in Algorithm~\ref{Algo3} is a class representative. 
We also define the state of the Algorithm as $S\in \mathcal{S}$, where
\begin{equation}
\begin{array}{lcl}
\mathcal{S}&=&\{S_{0,0},S_{1,0},S_{1,1},S_{2,0},S_{2,1},S_{2,2},\hdots,\\
&&S_{N-2,0},S_{N-2,1},\hdots,S_{N-2,N-2}\}
\end{array}
\end{equation}
and the state $S_{i,j}:\equiv E_{i,j}$. In other words, 
\begin{equation}
P(S_{i,j})=\text{Pr}\{E_{i,j}=1\}.    
\end{equation}
The total number of states is $N\!S=N(N-1)/2$. We also consider a vectorized version of the set of states $\mathcal{S}$ as ${\bf S}$, and when needed denote the state in time $t$ as $S^t \in \mathcal{S}$. The probability matrix of state transitions is then given by $\Pi := {\big [ P(S^t|S^{t-1}) \big ]}$. To derive the average query rate of Algorithm 3, we need a few computational prerequisites, which are presented in the next subsection. We complete the analysis in Section~\ref{Algo3Analysis}.

\subsubsection{States, Transitions and State Probabilities}\label{StateSec}
To setup the analysis, we formulate a state diagram of the Algorithm and aim to compute the stationary state probabilities. To this end, we first crystalize the outputs and transitions between the states and then derive the transition probability matrix $\Pi$. We can then obtain $P({\bf S})$ by solving a system of linear equations considering the following 
\begin{equation}
P({\bf S}^{t}) = P({\bf S}^{t-1})\times \Pi      
\end{equation}
in the steady state when $P({\bf S}^{t})=P({\bf S}^{t-1})$ and taking into account that the sum of state probabilities are one.

The transition probability is affected by the cases a, b, c, d, and e that may occur in a query, as described in Section~\ref{Algo3Desc} and Algorithm~\ref{Algo3}. The Algorithm output in each state is summarized below. Recall that in state $S_{i,j}$ two samples $s_1$ and $s_2$ with lengths $\ell(s_1)=i$ and $\ell(s_2)=j$, $0\leq j \leq i < N-1$, participate in a query with a class $s$. 

\begin{itemize}
\item 
	Case a (all three match): Two samples are labeled with length $i+1$ and $j+1$, no sample advances
\item 
	Case b ($s$ and $s_1$ match): One sample is labeled with length $i+1$, the sample with length $j+1$ advances 
\item 
	Case c ($s$ and $s_2$ match): One sample is labeled with length $j+1$, the sample with length $i+1$ advances
\item 
Case d ($s_1$ and $s_2$ match): One sample is labeled with length $j+1$, the sample with length $i+1$ advances
\item 
	Case e (no match): No labels, the sample with length $i+1$ advances and the sample with length $j+1$ is stored in the temporary bin
\end{itemize}

From the above description, it is now evident that we can only transition from $S_{i,j}$ to $S_{i+1,l}$ or $S_{l,j+1}$ or $S_{l,0}$; $0\leq l \leq N-2$. Here the sample of length $l>0$ comes from the temporary bin. The length being zero corresponds to the case where there is no sample in the temporary bin and we draw a new sample. We can now compute $\Pi$ as reported in the appendix. Since $\Pi$ is symmetric, only the lower triangular half of the matrix is presented.
As evident, the transition probabilities in $\Pi$ depend on the probability of a sample with a given length being in the temporary bin. They also depend on the unknown probabilities of states. So, one may invoke an iterative solution to this system of equations, starting from an initial value for $P(S_{i,j})$.

\subsubsection{Analysis of Query Rate}\label{Algo3Analysis}
We are now able to complete the analysis and obtain the query rate of the Algorithm~\ref{Algo3}.

\begin{theorem}\label{Algo3Theorem}
The average query rate of Algorithm~\ref{Algo3} for a uniformly distributed dataset is given by:
\begin{equation}
R_3^{kIC}=1/2 \sum_{q=1}^{N-1}q \times P(Q=q)    
\nonumber
\end{equation}
where for $1\leq q \leq N-1$
\begin{equation}\label{Event1Eq}
\begin{array}{rcl}
P(Q=q) &=& \prod_{k=1}^{q-1} \big(1-P(s \in \mathcal{L}_{k}|s \notin \mathcal{L}_1,\hdots, s \notin \mathcal{L}_{k-1})\big)\\
&\times&P(s \in \mathcal{L}_q | s_1 \notin \mathcal{L}_1, \hdots, s \notin \mathcal{L}_{q-1})
\end{array}
\end{equation}
and
\begin{equation}
\begin{array}{l}
P(s \in \mathcal{L}_k | s \notin \mathcal{L}_1, \hdots, s \notin \mathcal{L}_{k-1}) = \\
\sum_{i=0}^{N-2} P(S_{i,k-1}) \times P(s \in \mathcal{L}_{k}|{E}_{i,k-1}=1)
\end{array}
\end{equation}
and
\begin{equation}\label{ProbSuccessEq}
\begin{array}{ll}
P(s\in \mathcal{L}_{j+1}|{E}_{i,j}=1)=&\\
\begin{cases}
  \begin{cases}
p_b (i,j)+p_a(i,j) & i<j \\
p_c(i,j)+0.5p_d(i,j)+p_a(i,j) & i=j\\
p_c (i,j)+p_d(i,j)+p_a(i,j) & i>j   
  \end{cases}
  & 0 \leq j < N-2\\
1 & j=N-2
\end{cases}&
\end{array}
\end{equation}
and $P(S_{i,j})$ is obtained as described in Section~\ref{StateSec}.
\end{theorem}
\begin{proof}
The proof is provided in the appendix.
\end{proof}
\begin{remark}
The description and analysis of Algorithm~\ref{Algo3} presented so far focuses on labeling a uniform dataset. We can improve this algorithm for handling of nonuniform datasets as described below.
\begin{itemize}
    \item 
In Step 1.b: Sort the classes in decreasing order of their sizes
\item
In Step 3: 
\begin{itemize}
    \item 
In case a, re-start from the first class (bin)
\item
In cases b and c,
\begin{itemize}
    \item
If temporary bin for the next class is not empty, continue to the next class as usual
\item
If temporary bin for the next class is empty, drop the advancing sample in this bin, re-start from the first class
\end{itemize}
\item
In case d, act as usual and proceed to the next class
\end{itemize}
\end{itemize}
Again, the motivation here is to enhance the likelihood of labeling samples in a query.
\end{remark}

\section{Performance Evaluation}\label{PerfSec}
\begin{figure*}[t!]
\centering
\begin{minipage}[b]{0.45\linewidth}
\includegraphics[width=\textwidth]{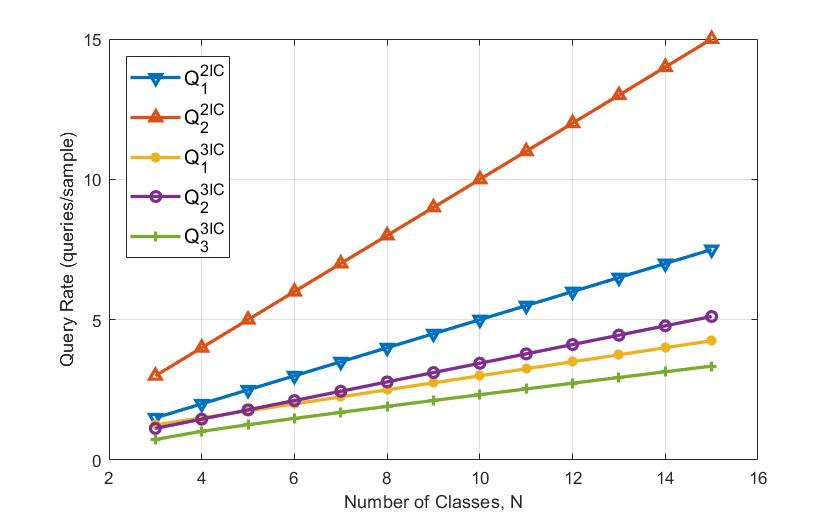}
\subcaption{}
\label{FigKICVSN}
\end{minipage}
\qquad
\begin{minipage}[b]{0.45\linewidth}
  \vspace*{\fill}
  \centering
  \includegraphics[width=\textwidth]{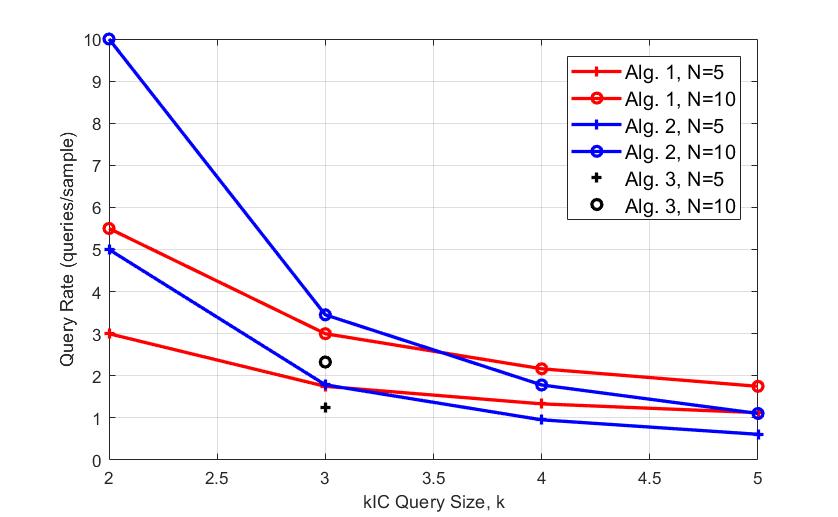}
  \subcaption{}
  \label{FigKICVSk}\par\vfill
\end{minipage}
\caption{\small{Query rate of the proposed algorithms for labeling of a dataset with uniform class distribution (a) as a function of number of classes $N$ for $2$IC and $3$IC query schemes, (b) as a function of number of samples $k$ in a $k$IC query for $N=5, 10$.}} \label{FigCompKIC}
\end{figure*}
Figure \ref{FigCompKIC} shows the query rate of the proposed schemes for labeling of a dataset with uniform class distribution and $k$IC queries. As evident, in case of $3$IC, the greedy adaptive Algorithm 3 performs best with a query rate of $\approx 0.2N$. The Algorithm 2 is appealing since it is randomized and runs in batches, however, in terms of average query rate it only surpasses that of the basic sample-by-sample Algorithm 1 for $k\ge 4$. This is evident both in Fig. \ref{FigKICVSk} and from Corollaries \ref{Algo1Cor1} and \ref{Algo2Cor1}. 
Specifically the Corollary \ref{Algo2Cor1} shows an appealing quadratic improvement of query rate with $k$ for Algorithm~\ref{Algo2}. 

\begin{figure*}[t!]
\centering
\begin{minipage}[b]{0.45\linewidth}
\includegraphics[width=\textwidth]{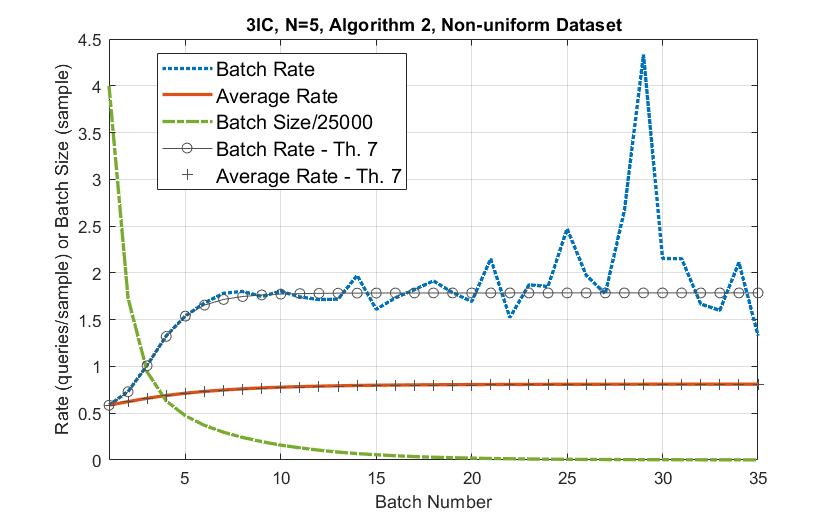}
\subcaption{}
\label{FigAlgo2a}
\end{minipage}
\qquad
\begin{minipage}[b]{0.45\linewidth}
  \vspace*{\fill}
  \centering
  \includegraphics[width=\textwidth]{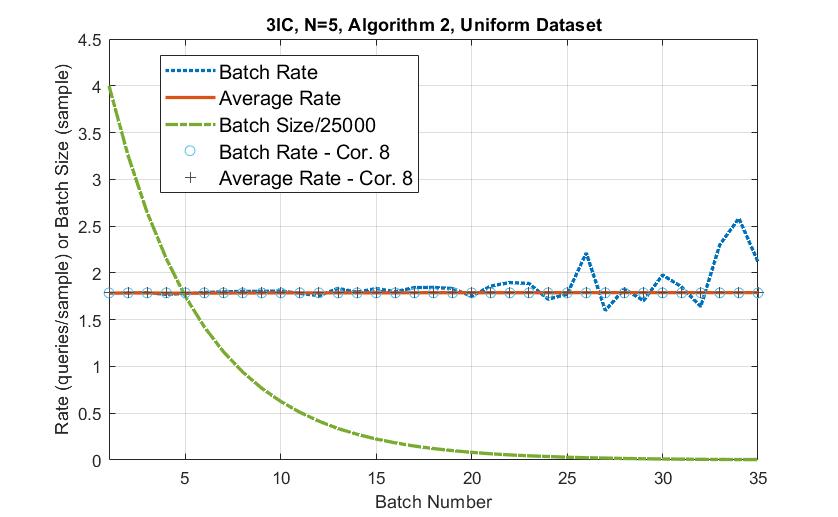}
  \subcaption{}
  \label{FigAlgo2b}\par\vfill
\end{minipage}
\caption{Performance of Algorithm 2 with batch query processing for $L=100000$, (a) $3$IC with nonuniform class distribution  $\boldsymbol{\pi}^n(10)=\{0.9\,\, 0.025\,\, 0.025\,\, 0.025\,\, 0.025\}$; (b) $3$IC with uniform class distribution}\label{FigAlgo2Performance}
\end{figure*}
Figure \ref{FigAlgo2Performance} demonstrates the performance of Algorithm 2 with batch processing for datasets with either a uniform or nonuniform class distribution $\boldsymbol{\pi}^n$ in \eqref{NonuniformClassEq2} for $x=10$. One sees that the batch size quickly shrinks (large number of samples settle) in the early rounds of the Algorithm. This is much more noticeable with nonuniform class distribution, since it is more likely that the two samples match in a $k$IC query when we have large cluster sizes. Interestingly, starting with a nonuniform class distribution, as the Algorithm proceeds, the class distribution of the remaining unlabeled samples evolves to a uniform distribution. This is both evident when we compare the batch rate in Figs. \ref{FigAlgo2a} and \ref{FigAlgo2b} and also from Corollary \ref{Algo2Cor2}. The results in section \ref{Algo2Sec} apply to the cases with large batch sizes, since the analysis relies on the effectiveness of the law of large numbers. As such, as soon as the number of (unlabeled) samples in the batch drop below a few hundred, the batch rate deviates from that predicted by Theorem \ref{Algo2Theorem} in a rather random manner. However, because of the small batch sizes at this stage of the Algorithm, this hardly affects the average query rate over the batches, which is accurately predicted by the said theorem.

Theorems \ref{Algo1Theorem} and \ref{Algo2Theorem} present closed form expressions of the query rates of Algorithms \ref{Algo1} and \ref{Algo2} for labeling a dataset with an arbitrary class distribution. Specializations for a uniform distribution and particular nonuniform distributions are presented in subsequent corollaries. Theorem \ref{Algo3Theorem} presents the query rate of Algorithm \ref{Algo3} that can be computed numerically for a dataset with uniform class distribution. 
Table \ref{tab:summary_results} summarizes the results of the presented analyses for the query rates of the Algorithms in selected cases of interest. As evident, the distribution of dataset plays a critical role. For the nonuniform distribution $\boldsymbol{\pi}^n(\frac{1}{\epsilon})$, the query rate of Algorithm~\ref{Algo1} is $(k-1)+0.5N\epsilon$ times that of Algorithm~\ref{Algo2}. For a uniform dataset, the said rate ratio is $0.25k$, showing advantage for Algorithm~\ref{Algo2} only for $k$'s after $4$. The advantage grows linearly with $k$ in both cases. Algorithm~\ref{Algo3} is presented for $k=3$. The analysis for uniform distribution clearly shows its query rate advantage.

We ran an experiment with the proposed Algorithms and the Stanford Dogs Dataset \cite{khosla2011novel}. The dataset we used contains $300$ images from six dog breeds. The Algorithms assume an Oracle, an expert on the subject. Volunteers participated in groups of five to label the images using each of the Algorithms for $k=2$ and $k=3$.  We also tested Algorithm~\ref{Algo1} with $k=N+1=7$. The volunteers were first presented with five sample images from each class for better understanding of the classes, prior to the tests. In each test, each of the participants were presented with queries constructed based on the Algorithms and $50$ random images drawn uniformly from the dataset. We made a few observations: (1) The query rate empirically computed in these experiments matches the presented analysis expectedly well. As the analysis of Algorithm~\ref{Algo2} relies on large batch sizes and we used only $50$ images for brevity, the empirical rate slightly overestimates what the analysis predicts. (3) Training of the participants greatly affected their accuracy. Their level of interest in the experiment also plays a role. With our non-expert participants average class error rates of about $10\%$ was typical. (4) The time it takes for handling a triplet query ($3$IC) is almost $50\%$ higher than that for a pairwise query ($2$IC). That is even though going from $2$IC to $3$IC involves three times more comparisons (querying three edges as seen in Figure \ref{3ICFig} as opposed to one), the effort required essentially increases by the number of new images added.  This is in line with the observations made in \cite{Ramya88}; observing and understanding the images (and hence their number) is the determining factor in the required time and not inputting the results. (5) We ran a test with Algorithm~\ref{Algo1} and $k=N+1=7$. The processing time was again almost proportionate, but somewhat less than the prediction. The decrease may be attributed to the fact that the $N$ class representatives remain constant in all our queries in Algorithm~\ref{Algo1}. We note that drawing broad conclusions in this regard demands extensive future research, that is well beyond the scope of the current work.

\begin{table*}[t]
    \centering
    \begin{tabular}{|c|c|c|c|c|c|}
         \hline
         Scheme & Alg. 1 $\boldsymbol{\pi}^u$ &  Algo. 1 $\boldsymbol{\pi}^n$ & Alg. 2 $\boldsymbol{\pi}^u$ & Alg. 2 $\boldsymbol{\pi}^n$ & Alg. 3 $\boldsymbol{\pi}^u$\\ \hline 
         Reference & Corr. 1 & Corr. 2 & Corr. 3 & Corr. 4 & Th. 3\\ \hline \hline
         $k=2$ & $\frac{N+1}{2}$ & $\cong 1+\frac{N\epsilon}{2}$ & $N$ & $\cong 1$&\\ \hline
         $k=3$ & $\frac{N+2}{4}$ & $\cong 1+\frac{N\epsilon}{4}$ & $\frac{N^2}{3N-1}$& $\cong \frac{1}{2}$&$\approx 0.2N$\\ \hline
         $k$ & $\frac{N+k-1}{2(k-1)}$ & $\cong 1+\frac{N\epsilon}{2(k-1)}$ & $\simeq\frac{2N}{k(k-1)}$ & $ \cong \frac{1}{k-1}$&\\ \hline
    \end{tabular}
    \caption{Summary of query rate analyses results for $k$IC with uniform class distribution $\boldsymbol{\pi}^u$ and nonuniform class distribution  $\boldsymbol{\pi}^n(\frac{1}{\epsilon}$). Results are for $2\leq k< N$ and in Algorithm~\ref{Algo1} for $k-1|N$. $\cong$ means equality up to $O(N\epsilon)$, $\approx$ is numerically approximately equal, $\simeq$ means asymptotically equal for $N\gg k>1$.}
    \label{tab:summary_results}
\end{table*}

\section{Concluding Remarks}\label{ConcSec}
This paper studied efficient ways of labeling a dataset using $k$IC queries from an oracle when no similarity function is available to relate the objects. In a $k$IC query, $k$ objects are posed at once, and it is expected that the respondent identifies the similar ones in return. The problem is important in machine learning from multiple perspectives: (1) Curating training datasets is a key step in design of supervised learning algorithms, which is in general a time consuming and costly process; (2) In clustering problems, such as entity resolution, a similarity function is often used to strategize, which samples to query (from crowd or an oracle). The similarity functions are typically application specific. As such, the (application agnostic) algorithms and results in this paper serve as benchmarks of what labeling efficiency can be achieved when no similarity function is indeed available; (3) In the same direction, the algorithms presented here can be used as the initial stage of an active learning scheme, when the data is still unknown and no similarity function is yet available. Similar situation could arise when dealing with subjective issues.

The summary of contributions and the insights obtained in this paper is as follows: We presented three algorithms for efficient labeling of a dataset using $k$IC queries to an oracle when no similarity function is available. A direct query for labeling of an object would require the respondent to provide the class it belongs to among $N$ possible classes (equivalent to $k$IC with $k=N+1$). In contrast, a $k$IC query (with $k<N$) operates based on comparisons and identifying (dis)similar objects in a set of $k$ objects. This provides a set of design possibilities for efficient query schemes that take advantage of the transitive relations and are potentially less (cognitively) complex and/or more reliable. This is visibly evident in cases with large number of classes, $N$, as it typical for example in entity resolution applications. Lemma~\ref{TheLemma} quantified the redundancy that is captured by exploiting transitive relations using $k$IC queries. The result presented in this Lemma may be of independent interest in the contexts of graph theory or number theory. The first presented algorithm in this work is a basic sample-by-sample scheme, whereby each sample is queried with $k-1$ class representatives at once. The second algorithm is a batch algorithm which operates in rounds with minimal complexity, in the sense that its storage complexity is simply the same minimum $O(L)$ complexity required to store the collected labels. The third algorithm is an adaptive greedy algorithm, which not only provides a superior query rate, but it also lends itself well to online applications. Specifically as an example, for a dataset of samples with uniform distribution from $N=10$ classes, Algorithm~\ref{Algo3} reduces the query rate by more than $40\%$, when compared to the basic Algorithm~\ref{Algo1} or the batch Algorithm~\ref{Algo2}. 
We analyzed the performance of the algorithms for uniform and certain types of nonuniform class distributions (when possible). While the analysis of the first two algorithms led to closed-form expressions, that for Algorithm~\ref{Algo3} provided a numerically computable solution. The techniques used for the latter may be of independent interest.  

A particular nonuniform class distribution of interest is one where there is a class with substantially larger size and the rest of the classes have similar sizes. This is an abstraction of a distribution with power law decay for which the presented analysis provides interesting insights. Such distributions are of high practical interest as many real datasets, e.g., income data, follow a Pareto-type distribution. Specifically our results indicate the importance of ordering of class (representatives) in sample-by-sample query Algorithms~\ref{Algo1} and \ref{Algo3}. For Algorithm~\ref{Algo1}, proper ordering of classes in comparison to the opposite can change the query rate order from $O(N)$ to $O(1)$. The results also showed that the proposed batch algorithm can exploit the nonuniformity effectively, with an advantage over Algorithm\ref{Algo1}, which grows linearly with $k$. Detailed numerical and simulation results were provided which shed light on the effectiveness of the proposed schemes and the accuracy of the analysis.

This research may be continued in a number of ways: (1) It is interesting to see how a similarity function may be used in conjunction with the proposed algorithms and how it affects their performance; (2) It was assumed that the $k$IC queries are posed to an oracle which makes no error. As such the focus in this work remained on design of efficient query schemes. It is of interest to develop schemes for inferring reliable labels from potentially erroneous query responses from a crowd. In this setting, the transitive relations exploited in $k$IC queries boost the reliability \cite{Lahouti16}. The inference schemes presented in e.g., \cite{DawidSkene79}\cite{karger}\cite{Vinayak2014}\cite{Yuchen14}\cite{Perona2010} would be informative starting points; (3) Our analysis showed the superior query rate efficiency of the Algorithm~\ref{Algo3} proposed for 3IC. It would be interesting to explore how it may be generalized to $k$IC for $k>3$. (4) In the current work, $k$IC queries solicit (dis)similarities among samples in an abstract manner. In practice, we often deal with multi-attribute classification problems. In such settings, it may be advantageous to explicitly incorporate this within the query schemes. Some of the attributes such as color may be easily annotated by machines, where other perhaps more subjective attributes may be better handled by humans. Hybrid labeling schemes involving machines and human annotations is an interesting research avenue. The joint distribution of attributes would also become useful in this setting. The interested reader is referred to \cite{Perona2010} for one related study. (5) The focus of the current research was on design and theoretical analysis of labeling schemes that are efficient from a query rate perspective. To arrive at a complete understanding of the overall budget for labeling, for example when preparing a dataset for training in machine learning applications, one needs to also take into account the query price. In general to set the query price, multiple factors should be considered. First the application domain, the business model and practical constraints set a framework. Whether the annotators are experts for hire, employees or short-term freelancers or simply enthusiasts play a critical role. The pricing solutions may be formulated as fixed or variable (adaptive) plans. The commercial services in this space have already adopted a variety of business models and specialize in different application domains. Second, one can take a cost perspective to pricing, as handling different queries may require different effort levels. The query cost may be quantified by the time it takes an annotator to process it (labeling time), and naturally not all annotators are the same. The (cognitive) difficulty or the level of expertise required for an annotation task are also important factors in understanding the query cost. Third, from a value perspective, again not all queries are equal, which we can consider in pricing. From an information theoretic view, in general we may gain more or less information (value) from a query when responded. The application domain and the quality and speed of the annotators also affect the value we get from a query response. This clearly shows the wide ranging and potentially intriguing innovation opportunities for research and experiments in this direction.   

\section*{Acknowledgement}
The authors wish to acknowledge O. Shokrollahi and the participants who helped with the experiments reported in Section~\ref{PerfSec}.


\bibliographystyle{IEEEtran}
\vskip 0.2in
\bibliography{Query22}

\appendices

\ifCLASSOPTIONcompsoc
\IEEEraisesectionheading{\section{Proof of Lemma 1}\label{AppLemma}}
\else
\section{Proof of Lemma 1}\label{AppLemma}
\fi
\IEEEPARstart{T}{he} problem resembles those of number composition and knapsack problems, albeit over a graph. The vertices are indexed by $j\in \{1, \hdots,k\}$. Every item in the set $\{1,\hdots,k\}$ is placed in exactly one bin (knapsack or clique) and exactly once. With $k$ vertices and $i$ disjoint subgraph cliques, the number of possible graph realizations is denoted by $g(k,i)$. 
This number when we limit ourselves to instances with the largest clique of size $j$ is $h(k,i,j)$. We have
\begin{equation}\label{gFunDef2}
g(k,i)=\sum_{j=\lceil{\frac{k}{i}}\rceil}^{k-i+1} h(k,i,j)
\end{equation}
With $k$ items and $i$ bins, if we put one item in every bin except the largest one, we arrive at $k-i+1$ as the maximum number of items in the largest bin. The minimum number of items in the largest bin is $\lceil{\frac{k}{i}}\rceil$. This is when the items are distributed in the bins almost uniformly. This gives $h(k,i,j)$ in \eqref{hEQ2}, where $h(l,l,1)=1, h(l,1,l)=1, g(l,1)=1, g(l,l)=1, \forall l\ge 1$ and $h(k,i,1)=0, \forall k \neq i$.

\begin{figure*}
\begin{equation}\label{hEQ2}
h(k,i,j)=
\begin{cases}
{k \choose j} \times \sum_{l=\lceil{\frac{k-j}{i-1}}\rceil}^{\min(j,k-j-i+2)} \frac{1}{1+\rho^*(k,i-1,j,l)}\times h(k-j,i-1,l)& \lceil{\frac{k}{i}}\rceil \leq j \leq k-i+1\\
0 & \text{otherwise}
\end{cases}
\end{equation}
\end{figure*}
The number of ways we can choose $j$ items, hence forming the largest bin, in the set of $k$ items is ${k \choose j}$. Once the largest bin with $j$ items is chosen, then we have $k-j$ remaining items and $i-1$ bins. The summation in the definition of $h$, assuming $\rho^*(k,i,j,l)=0$, is equal to $g(k-j,i-1)$. The range of summation in \eqref{hEQ2} is obtained in the same manner as we did for \eqref{gFunDef2} and noting that $l\le j$.
In some occasions, we have more than one bin with the largest number of items. The factor $\frac{1}{1+\rho^*}$ is meant to count for that. With $k-j$ items and $i$ bins whose biggest one has $l$ items, $\rho(k,i,j,l)$ is the number of bins with size $j$ that we can have. We have $l\le j$, and from the definition $\rho(k,i,j,l)=0, \forall l<j$. We have
\begin{equation}\label{EqRhoOpt}
\begin{array}{ccl}
    \rho^*(k,i,j,j)&=&\max_{\rho \in \mathbb{Z}^+}\rho \\
    &&\\
    \rho & \le & \frac{k-j-i}{j-1}\\
    \rho & \le & \lfloor{\frac{k-j}{j}}\rfloor
    \end{array}
\end{equation}
where $\mathbb{Z}^+$ indicate the non-negative integers.The first condition in \eqref{EqRhoOpt} ensures we have the $i$ non-empty bins (with at least one item in each of them), i.e.,
\begin{equation}\label{FirstCondRho}
k-j-\rho j \ge i-\rho \Rightarrow \rho \le \frac{k-j-i}{j-1}.    
\end{equation}
In the remaining $k-j$ items, we can at most have $\lfloor{\frac{k-j}{j}}\rfloor$ bins with the size $j$ which leads to the second condition in \eqref{EqRhoOpt}. The solution to the optimization problem \eqref{EqRhoOpt} then gives
\begin{equation}\label{RhoSolutionApp}
    \rho^*(k,i,j,l)= 
    \begin{cases}
    \lfloor{\min\big(\frac{k-j-i}{j-1},\frac{k-j}{j}\big)}\rfloor & j=l\\
    0 & otherwise 
    \end{cases}
\end{equation}
for the non-trivial cases of interest, $i,j,k\ge 2$, $k > j \ge l$, and $i \le k$. Note that if \eqref{FirstCondRho} is to be true for any $\rho>0$ it is definitely true for $\rho=0$ and hence $k-j-i>0$, and $\rho^*\ge 0$ in \eqref{RhoSolutionApp}. 
The function $g(k,i)$ for $2\ge k\ge 10$ and $i\le k$ is computed as in Table \ref{tab:my_label}. Figure \ref{FigfkN} shows $f(k,N)$ as a function of $N$.
\begin{figure}[ht]
\includegraphics[width=\linewidth]{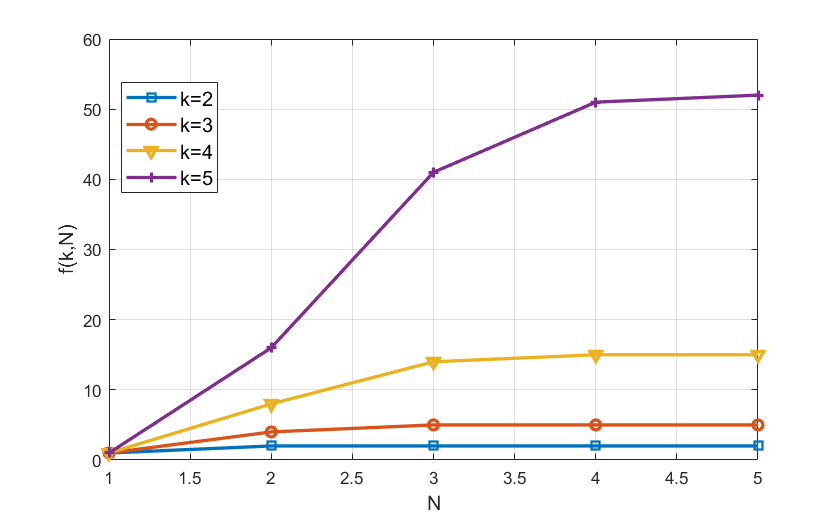}
\caption{$f(k,N)$ as a function of $N$}
\label{FigfkN}
\end{figure}
\begin{table*}[ht]
    \centering
    \begin{tabular}{|c||r|r|r|r|r|r|r|r|r|}    
    \hline
\diagbox[]{$k$}{$i$} & 2 & 3 & 4 & 5 & 6 & 7 & 8 & 9 & 10\\\hline\hline
2 & 1	& 0	& 0	& 0	& 0	& 0	& 0	& 0	& 0\\\hline
3 & 3	& 1	& 0	& 0	& 0	& 0	& 0	& 0	& 0\\\hline
4 & 7	& 6	& 1	& 0	& 0	& 0	& 0	& 0	& 0\\\hline
5 & 15	& 25 & 10	& 1	& 0	& 0	& 0	& 0	& 0\\\hline
6 & 31	& 90 & 65 & 15 & 1 & 0 & 0 & 0 & 0\\\hline
7 & 63	& 301	& 350	& 140	& 21	& 1	& 0	& 0	& 0\\\hline
8 & 127	& 966	& 1701	& 1050	& 266	& 28	& 1	& 0	& 0\\\hline
9 & 255	& 3025	& 7770	& 6951	& 2646	& 462	& 36	& 1	& 0\\\hline
10 & 511	& 9330	& 32005	& 42525	& 22827	& 5880	& 750	& 45	& 1\\\hline
\end{tabular}
\caption{$g(k,i)$: The number of possible graph realizations with $k$ vertices and $i$ disjoint subgraph cliques}
\label{tab:my_label}
\end{table*}

Here are a few examples for function $f(k,N)$:
\begin{equation}
f(k,N)=\sum_{i=1}^{\min{(N,k)}} g(k,i).    
\end{equation}
\begin{equation}
    \begin{array}{ccl}
         f(3,2)&=& g(3,1) + g(3,2)\\
         &=& h(3,1,3) + h(3,2, 2)\\
         &=& 1+{3\choose 2}=4\\
         
         f(4,2) &=& g(4,1)+g(4,2)\\
        &=& h(4,1,4) + h(4,2,2) + h(4,2,3)\\
         &=&1+{4 \choose 2} \times 0.5 \times h(2,1,2)+{4 \choose 3}=8.
         \end{array}
         \end{equation}
And here are a few examples for function $g(k,i)$:
\begin{equation}\label{gexampleEQ}
    \begin{array}{ccl}
         g(5,3)&= & h(5,3,2)+h(5,3,3)\\ 
         &=& {5 \choose 2}\times \frac{1}{1+\rho^*(5,2,2,2)}\times h(3,2,2)+ {5 \choose 3}\times \frac{1}{1+\rho^*(5,2,3,3)}\times h(2,2,1)\\
         &=&{5 \choose 2}\times 0.5 \times {3 \choose 2}+{5 \choose 3}=25\\
         
         g(6,3)&= & h(6,3,2) + h(6,3,3) + h(6,3,4)\\
         &=&{6 \choose 2}\times \frac{1}{1+\rho^*(6,2,2,2)} \times h(4, 2, 2) + {6 \choose 3}\times \frac{1}{1+\rho^*(6,2,3,2)} \times h(3, 2, 2)\\
         &+& {6 \choose 4}\times \frac{1}{1+\rho^*(6,2,4,1)} \times h(2, 2, 1)\\
         &=&{6 \choose 2}\times \frac{1}{3} \times h(4, 2, 2) + {6 \choose 3}\times 1 \times h(3, 2, 2) + {6 \choose 4}\times 1 \times h(2, 2, 1)\\
         &=&5\times 0.5 \times {4 \choose 2}+20 \times {3 \choose 2}+15\\
         &=&15+60+15=90\\
         \end{array}
\end{equation}

\section{Proofs of Section~3}
\label{Algo1App}


In Algorithm~1 with $k$IC, a given sample is queried with $k$ classes at a time. The probability that the sample finds a match with one of the classes in a query is then $\pi'=\sum_{j \in \mathcal{J}}\pi_j$. Here $\pi_j$ is the probability of the class $j$ and $\mathcal{J}$ is the class representatives participating in the current $k$IC query with the current sample. Taking into account that in Algorithm~1 the set of classes participating in a query with a given sample is chosen in sequence, once and if the sample does not settle in preceding queries, we can write the probability $P(E_q)$ of sample $s$ settling in query $q$  as follows:
\begin{equation}\label{Eq40}
\begin{array}{ccl}
P(E_q) &=& P(s \notin \mathcal{L}_1,\hdots, s \notin \mathcal{L}_{q-1},s \in \mathcal{L}_q)\\
&=&P(s \notin \mathcal{L}_1,\hdots, s \notin \mathcal{L}_{q-1})\\
&\times& P(s \in \mathcal{L}_{q} | s \notin \mathcal{L}_1,\hdots, s \notin \mathcal{L}_{q-1})\\
&=& (1-\sum_{i=1}^{q-1} \pi'_i) \times \frac{\pi'_q}{\sum_{i=q}^{\frac{N}{k-1}} \pi'_i}\\
&=&\pi'_q
\end{array}
\end{equation}
In \eqref{Eq40}, without loss of generality we assume $k-1|N$, and $\pi'_i, 1\leq i \leq \frac{N}{k-1}$ represents the sum of the probabilities of the $k-1$ classes that participate in each of the $k$IC queries. 
The expected number of queries it takes for the sample to settle is then given by
\begin{equation}\label{Eq41}
R_1^{kIC}=\sum_{i=1}^{\frac{N}{k-1}}i \pi'_{i} \hspace{1cm} \text{queries/sample}
\end{equation}
where $2\leq k <N$. Consider two specific ordering of the class representatives for query in Algorithm~1; the first one is ${\boldsymbol{\pi}^1}=\{\pi_1, \hdots, \pi_{j-1}, \pi_j,\pi_{j+1}, \hdots, \pi_{l-1},\pi_l,\pi_{l+1}, \hdots, \pi_N\}$ and the second one is ${\boldsymbol{\pi}^2}=\{\pi_1, \hdots, \pi_{j-1}, \pi_l,\pi_{j+1}, \hdots, \pi_{l-1},\pi_j,\pi_{l+1}, \hdots, \pi_N\}$. As evident, the difference between the two is that the positions of $\pi_j$ and $\pi_l$ have been switched. Now, as evident in \eqref{Eq41}, this switching only affects the rate if $j$ and $l$ are originally positioned in two different $k$IC queries and $\pi_j\neq \pi_l$. In this case, if based on $\boldsymbol{\pi}^1$, $\pi_j$ and $\pi_l$ appear in $i$th and $i'$th queries, respectively and $\pi_l<\pi_j, i<i'$, we have
\begin{equation}
\begin{array}{ccl}
R_1^{kIC}(\boldsymbol{\pi}^1)-R_1^{kIC}(\boldsymbol{\pi}^2)&=&i\times (\pi_j-\pi_l) + i' \times (\pi_l - \pi_j)\\
  &=   & (i-i')\times(\pi_j-\pi_l) < 0
\end{array}
\end{equation}
This shows that if we order the class representatives in Algorithm~1 in reducing order of their class sizes, we achieve the smallest query rate. Similarly, one sees that if we order the class representatives in Algorithm~1 in increasing order of their class sizes, it leads to the largest query rate.

In the case of uniform class distribution, we have $\pi'_i=(k-1)/N, 1\leq i \leq N/(k-1)$ and hence we have 
\begin{equation}
R_1^{kIC}(\boldsymbol{\pi}^u)= \sum_{i=1}^\frac{N}{k-1}i\,(k-1)/N
\end{equation}
which gives the result in Corollary~1.

To obtain the results in Corollary~2 with $k$IC and assuming $k-1|N$, we consider the nonuniform distribution, $\boldsymbol{\pi}^n(x)$, presented in Section~2 of the paper. We examine two scenarios, when the class with distinct size is positioned either as the \emph{first} one or as the very \emph{last} one for query in Algorithm~1. In the former scenario, for average query rate, we have
\begin{equation}
    R^{kIC}_{1,f}= \frac{\alpha}{N}+(k-2)\frac{\alpha'}{N}+\sum_{i=2}^{\frac{N}{k-1}}i(k-1)\,\frac{\alpha'}{N}
\end{equation}
adding and subtracting a $(k-1)\frac{\alpha'}{N}$ term and taking into account the description of $\boldsymbol{\pi}^n$ in Section~2 of the paper, we obtain
\begin{equation}
\begin{array}{cl}
    R^{kIC}_{1,f}(\boldsymbol{\pi}^n(x))=&\frac{N(N+k-1)}{2(N-1)(1+x)(k-1)}\\
    &+\frac{x}{1+x}-\frac{1}{(1+x)(N-1)}
\label{xFirstEq}
\end{array}
\end{equation}
and in the latter scenario, for average query rate, we have 
\begin{equation}
    R^{kIC}_{1,l}(\boldsymbol{\pi}^n(x))= \sum_{i=1}^{\frac{N}{k-1}-1}i(k-1)\frac{\alpha'}{N}+\Big(\frac{N}{k-1}\Big)\Big(\frac{\alpha}{N}+(k-2)\frac{\alpha'}{N}\Big)
\end{equation}
and following some mathematical manipulations, we have
\begin{equation}
    R^{kIC}_{1,l}(\boldsymbol{\pi}^n(x))= \frac{N}{(k-1)(1+x)}\Big(x+\frac{N+k-3}{2(N-1)}\Big).
\label{xLastEq}
\end{equation}
With \eqref{xFirstEq} and \eqref{xLastEq}, we can now examine the two cases when the class with the distinct size is either very small or very large as described in Section~2 of the paper. For $\epsilon\ll 1$ and large $N$, and $N\epsilon \ll 1$, if we set $x=\epsilon$, in $R_{1, f}$ and $R_{1, l}$, we obtain
\begin{equation}
\frac{N(1+\epsilon)+k-3}{2(k-1)}\lesssim R_1^{kIC}\left(\boldsymbol{\pi}^n(\epsilon)\right)\lesssim \frac{N+k-1-N\epsilon}{2(k-1)} 
\end{equation}
where $\lesssim$ means $\leq$ up to $O(N\epsilon)$. If instead we set $x=1/\epsilon$ in the said equations, for $\epsilon\ll 1$ and large $N$, we obtain
\begin{equation}
1+\frac{N\epsilon}{2(k-1)}\lesssim R_1^{kIC}\left(\boldsymbol{\pi}^n\left(1/\epsilon\right)\right) \lesssim \frac{N(2-\epsilon)}{2(k-1)}. 
\end{equation}
This completes the proof.

\section{Proofs of Section~4}
\label{Algo2App}
In this Section, we present the proof of Theorem~2. For simpler presentation, we take a brief lapse in the notations within this proof that is clear from the text.
We consider a single round of Algorithm~2 with one large batch of data with $L_1$ samples and class probabilities $\boldsymbol{\pi}$ presented in Section~2 of the paper. In Algorithm~2, the queries are formed with samples drawn randomly from the batch. Consider the samples participating in a $k$IC query as the vector $\mathbf{X}=(X_1,\hdots,X_k)$. Following this query, $0,1,\hdots,k-2$ or $k-1$ samples may be settled and as a result, $k,k-1,\hdots,2$ or $1$ sample(s) are returned to the batch for the next round. We denote $Y\in \{1, \hdots, N\}\cup \{\ell\}$ as the random variable showing the outcome of the query. Indeed, $Y=\ell$ indicates the case where a sample is settled, i.e., $P(Y=\ell) \equiv P(s\in{\cal L})$. We have
\begin{equation}\label{PYEq}
P(Y)=\sum_{\mathbf{X}}P(Y|\mathbf{X})P(\mathbf{X}).    
\end{equation}
With $P(X=i)=\pi_i, {\bf I}=(i,\hdots,i)_{1\times k}, 1 \leq i \leq N$, and $d_H(.,.)$ denoting the Hamming distance of the two vectors in the argument, we have
\begin{equation}\label{EqPY}
\begin{array}{ll}
P(Y|X_1,\hdots,X_k)=&\\
\begin{cases}
0 & Y \neq X_j \\
\frac{1}{k} & Y=X_j, \hspace{0.5cm} j\in\{1,\hdots,k\},\\
\frac{[d_i-1]_+}{k} & Y=\ell 
\end{cases}
\end{array}
\end{equation}
where $d_i :=k-d_H (\mathbf{X},\mathbf{I}))$, and $[x]_+=0$ if $x<0$ and $[x]_+=x$ if $x\ge0$. As a result, 
\begin{equation}\label{EqPYi}
\begin{array}{lcl}
P(Y=i)&=&\frac{1}{k}\sum_{j=1}^k {\genfrac(){0pt}{0}{k}{j}} \pi_i^j (1-\pi_i )^{k-j}\\
&=&\frac{1}{k} \big(1-(1-\pi_i )^k \big)
\end{array}
\end{equation}
and since $P(Y=\ell)=1-\sum_{j=1}^{N}P(Y=j)$, we have
\begin{equation}\label{EqSettle}
P(Y=\ell)=1-\frac{1}{k} \sum_{j=1}^{N}\big(1-(1-\pi_j)^k \big). 
\end{equation}
Note that computing $P(Y=\ell)$ directly from \eqref{EqPY} is complicated, however, since we are able to obtain $P(Y=i)$ from \eqref{EqPYi}, we can compute $P(Y=\ell)$ as presented above in \eqref{EqSettle}.
The probability of a sample being in class $i\in\{1,\hdots, N\}$ is now obtained by normalizing $P(Y=i)$:
\begin{equation}\label{EqProb}
\pi'_i=\frac{P(Y=i)}{1-P(Y=\ell)}
\end{equation}
which gives 
\begin{equation}
\pi'_i=\frac{1-(1-\pi_i)^k}{N-\sum_{j=1}^{N} (1-\pi_j)^k}, \hspace{1cm} 1\le i \le N,
\label{EqPiEvolutionA}
\end{equation}
and completes the first part of proof. Note that throughout the proof and starting from \eqref{PYEq}, we utilize the fact that the dataset is large and as such the probability of samples in a random $k$IC query follows those from the dataset class distribution. For this single batch of large size $L_1$, $\frac{L_1}{k}$ queries are made, $L_1 P(Y=\ell)$ samples are settled, which leads to an average query rate of $R^{kIC}_{2,1} = \frac{1}{kP(Y=\ell)}$ (queries/sample) and a reduced batch size of $L_2=L_1(1-P(Y=\ell))$. 

Equation \eqref{EqPiEvolutionA} enables us to track the class probabilities over $m\ge 1$ consecutive rounds of Algorithm~2, as long as the law of large number remains in effect, i.e., $L_{m+1}\gg 1$. Also, in round $1 \le r \le m$, this allows us to compute the probability that a sample is settled, $P(Y_r=\ell)\equiv P(s\in {\cal L}_r)$, the batch size $L_r$, and the average query rate in the round, ${R}^{kIC}_{2,r}$ and over all $m$ rounds, $R^{kIC}_{2,1:m}$. We have
\begin{equation}
L_r=L \Pi_{i=1}^{r-1}\big(1-P(Y_{i}=\ell)\big)    
\end{equation}
\begin{equation}\label{EqSingleRoundRate}
{R}^{kIC}_{2,r}=\frac{1}{kP(Y_r=\ell)}     
\end{equation}
\begin{equation}\label{EqMultiRoundRate}
\begin{array}{ccl}
R^{kIC}_{2, 1:m} &=& \frac{1}{k}\Big(L+L\big(1-P(Y_1=\ell)\big)+\hdots\\
&&+L\Pi_{i=1}^{m-1}\big(1-P(Y_{i}=\ell)\big)\Big) \times\\
&&\Big(LP(Y_1=\ell)+L\big(1-P(Y_1=\ell)\big)P(Y_2=\ell)+\\
&\hdots&+L\Pi_{i=1}^{m-1}\big(1-P(Y_{i}=\ell)\big)  P(Y_m=\ell)\Big)^{-1}\\
&=&\frac{1}{k}{\sum_{r=1}^m \Pi_{i=1}^{r-1}\big(1-P(Y_{i}=\ell)\big)}\times\\
&&\Big({\sum_{r=1}^m \Pi_{i=1}^{r-1}\big(1-P(Y_{i}=\ell)\big)  P(Y_r=\ell)}\Big)^{-1}
\end{array}
\end{equation}
This completes the proof of Theorem~2. 

To derive the results presented in Corollary~3, we specialize the class distribution to uniform $\boldsymbol{\pi}=\boldsymbol{\pi}^u$. Replacing $\pi_i=\frac{1}{N}$ in \eqref{EqPiEvolutionA}, it is straight forward to see that indeed $\pi'_i=\frac{1}{N}$ and $\boldsymbol{\pi}'=\boldsymbol{\pi}^u$. In this case, using \eqref{EqSettle}, we have
\begin{equation}
P^u(Y=\ell)=1-\frac{N}{k}\big(1-(1-\frac{1}{N})^k \big),     
\end{equation}
and when $N\gg k>1$, we obtain
\begin{equation}
\begin{array}{lcl}
P^u(Y=\ell)&\simeq& 1-\frac{N}{k}\bigg(1-\big(1-\frac{k}{N}+\frac{k(k-1)}{2N^2}\big)\bigg)\\
&=&\frac{k-1}{2N}.    
\end{array}
\end{equation}
Since the class distribution is uniform and it remains as such through multiple rounds of Algorithm~2, the average query rate also remains the same. Using \eqref{EqSingleRoundRate} and \eqref{EqMultiRoundRate}, we have 
\begin{equation}
R^{kIC}_{2,1:m}(\boldsymbol{\pi}^u)={R}_r(\boldsymbol{\pi}^u)=\frac{1}{kP^u(Y=\ell)}\simeq \frac{2N}{k(k-1)}   
\end{equation}
for $1\le r \le m$.

To obtain the results in Corollary~4, we specialize the class probability distribution $\boldsymbol{\pi}$ to $\boldsymbol{\pi}^n(\frac{1}{\epsilon})$. In this case, we have ${\pi}_i=\frac{\epsilon}{(N-1)(1+\epsilon)}\approx \frac{\epsilon}{N-1}, 1\le i \le N-1$ and $\pi_N=\frac{1}{1+\epsilon} \approx 1-\epsilon$. For a large batch of samples from this class distribution, using \eqref{EqSettle}, we have 
\begin{equation}
\begin{array}{ccl}
P(Y=\ell) &=& 1-\frac{1}{k} \sum_{i=1}^N(1-(1-\pi_i)^k)\\
&\approx& 1-\frac{1}{k}\bigg((N-1)\big(1-(1-\frac{k\epsilon}{N-1})\big)+1\bigg)\\
&=&1-\frac{1}{k}(1+k\epsilon)=\frac{k-1}{k}-\epsilon
\end{array}
\end{equation}
And from \eqref{EqProb}, we obtain
\begin{equation}
\begin{array}{ccl}
P(Y'=N) &=& \frac{1-\epsilon^k}{N-\epsilon^k-(N-1)(1-\frac{\epsilon}{N-1})^k}\\ &\approx&\frac{1-\epsilon^k}{N-\epsilon^k-(N-1-k\epsilon)}\\
&\approx&\frac{1}{1+k\epsilon} \approx 1-k\epsilon
\end{array}    
\end{equation}
and
\begin{equation}
P(Y'=i) \approx \frac{k\epsilon}{N-1} \hspace{1cm}         i\in\{1,\hdots,N-1\}.    
\end{equation}
Using \eqref{EqSingleRoundRate}, the average query rate in a single round of Algorithm~2 with class probability $\boldsymbol{\pi}^n(\frac{1}{\epsilon})$ is then given by
\begin{equation}
\begin{array}{ccl}
R^{kIC}_2\big(\boldsymbol{\pi}^n(\frac{1}{\epsilon})\big) &=&\frac{1}{kP(Y=\ell)}=\frac{1}{k-1-k\epsilon}\\
 &\approx& \frac{1}{k-1}(1+\frac{k}{k-1}\epsilon)=\frac{1}{k-1}+\epsilon'      
\end{array}
\end{equation}
where $\epsilon':=\frac{k\epsilon}{(k-1)^2}$.



\section{Proofs of Section~5}
\subsection{Computing $\Pi$}
In case $0<i<N-2$ and $j<i$, we have the set of transitions in \eqref{Set1EQ};
\begin{equation}\label{Set1EQ}
\begin{array}{lcl}
P(S_{i+1,0}|S_{i,j}) &=& (p_c+p_d+p_e)\times P(\ell(s_t \in \mathcal{T}=0) + p_a \times P(\ell(s_t \in \mathcal{T})=i+1) \\
P(S_{l,0}|S_{i,j}) &=& p_a \times P(\ell(s_t \in \mathcal{T})=l), \quad 0\leq l \leq N-2, l \ne i+1;\\
P(S_{i+1,l}|S_{i,j}) &=& (p_c+p_d+p_e) \times P(\ell(s_t \in \mathcal{T})=l), \quad 0< l \leq i+1, l \ne j+1;\\
P(S_{l,i+1}|S_{i,j}) &=& (p_c+p_d+p_e) \times P(\ell(s_t \in \mathcal{T})=l), \quad i+1\leq l \leq N-2;\\
P(S_{i+1,j+1}|S_{i,j})&=&(p_c+p_d+p_e) \times P(\ell(s_t \in \mathcal{T})=j+1)+p_b \times P(\ell(s)_t \in \mathcal{T})=i+1)\\
P(S_{j+1,l}|S_{i,j})&=&p_b \times P(\ell(s _t \in \mathcal{T})=l), \quad 0< l \leq j+1;\\
P(S_{l,j+1}|S_{i,j})&=&p_b \times P(\ell(s _t \in \mathcal{T})=l), \quad j+1\leq l \leq N-2, l \ne i+1;\\
P(S_{j+1,0}|S_{i,j})&=&p_b \times P(\ell(s _t \in \mathcal{T})=0)+p_a \times P(\ell(s _t \in \mathcal{T})=j+1)
\end{array}    
\end{equation}
In case we start from state $(i,i)$, we have the set of transitions in \eqref{Set2EQ};
\begin{equation}\label{Set2EQ}
\begin{array}{lcl}
P(S_{i+1,0}|S_{i,i}) &=& (p_b(i,i)+p_c(i,i)+p_d(i,i)+p_e(i,i)) \times P(\ell(s_t \in \mathcal{T})=0)\\
&+& p_a(i,i) \times P(\ell(s_t \in \mathcal{T})=i+1)\\
P(S_{l,0}|S_{i,i}) &=& p_a (i,i) \times P(\ell(s_t \in \mathcal{T})=l), \quad 0\leq l \leq N-2, l \ne i+1;\\
P(S_{i+1,l}|S_{i,i}) &=& (p_b (i,i)+p_c (i,i)+p_d (i,i)+p_e (i,i)) \times P(\ell(s_t \in \mathcal{T})=l), \quad 0< l \leq i+1;\\
P(S_{l,i+1}|S_{i,i}) &=& (p_b (i,i)+p_c (i,i)+p_d (i,i)\\
&+&p_e (i,i)) \times P(\ell(s_t \in \mathcal{T})=l), \quad i+1< l \leq N-2;
\end{array}    
\end{equation}
In case $i=N-2, 0\leq j < N-2$, we have the set in \eqref{Set3EQ};
\begin{equation}
\begin{array}{lcl}\label{Set3EQ}
P(S_{l,0}|S_{N-2,j}) &=& p_f (j) \times P(\ell(s_t \in \mathcal{T})=l), \quad 0\leq l \leq N-2, l \ne j+1;\\
P(S_{j+1,0}|S_{N-2,j}) &=& p_f (j) \times P(\ell(s_t \in \mathcal{T})=0)+p_g (j) \times P(\ell(s_t \in \mathcal{T})=j+1)\\
P(S_{j+1,l}|S_{N-2,j}) &=& p_g (j) \times P(\ell(s_t \in \mathcal{T})=l), \quad 0< l \leq j+1;\\
P(S_{l,j+1}|S_{N-2,j}) &=& p_g (j) \times P(\ell(s_t \in \mathcal{T})=l), \quad j+1< l \leq N-2;\\
\end{array}    
\end{equation}
And finally,
\begin{equation}
\begin{array}{ll}
P(S_{l,0}|S_{N-2,N-2})=&\\ 
P(\ell(s_t \in \mathcal{T})=l), &\quad 0\leq l \leq N-2.
\end{array}
\end{equation}

As evident, the above transition probabilities depend on the probability of a sample with a given length being in the temporary bin. We compute this next in \eqref{lengthEQ1} and \eqref{lengthEQ2}. 
\begin{equation}\label{lengthEQ1}
\begin{array}{ccl}
P(\ell(s \in \mathcal{T})=l+1) &=& \sum_{m=0}^{N-2} P(s \in \mathcal{T}|E_{m,l})P((s,s') \in {\cal Q},\ell(s)=l,\ell(s')=m)\\ 
&=&\sum_{m=l}^{N-2} p_e(m,l) P(S_{m,l}); \hspace{1cm} 0\leq l \leq N-2
\end{array}
\end{equation}
\begin{equation}\label{lengthEQ2}
\begin{array}{ccl}
P(\ell(s \in \mathcal{T})=0) &=& \sum_{l=0}^{N-2} \sum_{m=0}^{N-2} P(s \notin \mathcal{T}|E_{l,m}) P((s,s') \in {\cal Q},\ell(s)=l,\ell(s')=m)\\ 
&=&\sum_{l=0}^{N-2}   \sum_{m=l}^{N-2} (1-p_e(m,l)) P(S_{m,l})
\end{array}
\end{equation}
One sees in the above formulation, the transition matrix $\Pi$ also depends on the unknown probabilities of states. So, one may invoke an iterative solution to this system of equations, starting from an initial value for $P(S_{i,j})$.

\subsection{Proof of Lemma~2}
The proof immediately follows noting how many samples are settled in any of the above five cases. Specifically, two samples are settled in case a, and one sample is settled in cases b, c, and d. Therefore, we have $SL=2p_a(i,j)+p_b(i,j)+p_c(i,j)+p_d(i,j)$ and then the proof follows using the probabilities of the cases obtained above. Note that in case d, two samples match (they are from the same class). As such, we would need to query only one of them (as representative) subsequently and we consider the other one as settled. 

\subsection{Proof of Theorem~3}
For a dataset with equal class sizes, the probability that the sample $s_1$ settles in the first query $q=1$ is given by
\begin{equation}
\begin{array}{ccl}
P(Q=1) &=& P(s_1 \in \mathcal{L}_1)\\
&=&\sum_{\ell(s_2)=i=0}^{N-2}
P(s_1\in \mathcal{L}_1|(s_1,s_2) \in \mathcal{Q},\\
&&\ell(s_1)=0,\ell(s_2)=i)\\
&\times& P((s_1,s_2)\in \mathcal{Q},\ell(s_1)=0, \ell(s_2)=i)\\
&=&\sum_{\ell(s_2)=i=0}^{N-2}P(s_1\in \mathcal{L}_1|{E}_{i0}=1)\\ 
&\times& P({E}_{i0}=1)\\
&=&\sum_{\ell(s_2)=i=0}^{N-2}P(S_{i,0}) \times P(s_1 \in \mathcal{L}_1|{E}_{i0}=1);
\end{array}    
\end{equation}
Similarly, for $q=2$, we have
\begin{equation}
\begin{array}{ccl}
P(Q=2) &=& P(s_1 \notin \mathcal{L}_1,s_1 \in \mathcal{L}_2)\\
&=&P(s_1 \notin \mathcal{L}_1) \times P(s_1 \in \mathcal{L}_2 | s_1 \notin \mathcal{L}_1)\\
&=& (1-P(s_1 \in \mathcal{L}_1))\\
&\times& \sum_{l(s_2)=i=0}^{N-2} P(S_{i,1}) \times P(s_1 \in \mathcal{L}_2|{E}_{i1}=1);
\end{array}
\end{equation}
And for $q=3$, we have
\begin{equation}
\begin{array}{ccl}
P(Q=3) &=& P(s_1 \notin \mathcal{L}_1,s_1 \notin \mathcal{L}_2,s_1 \in \mathcal{L}_3)\\
&=&P(s_1 \notin \mathcal{L}_1, s_1 \notin \mathcal{L}_2)\\ 
&\times& P(s_1 \in \mathcal{L}_3 | s_1 \notin \mathcal{L}_1, s_1 \notin \mathcal{L}_2)\\
&=&P(s_1 \notin \mathcal{L}_1) \times P(s_1 \notin \mathcal{L}_2|s_1 \notin \mathcal{L}_1)\\
&\times& P(s_1 \in \mathcal{L}_3 | s_1 \notin \mathcal{L}_1, s_1 \notin \mathcal{L}_2)\\
&=& (1-P(s_1 \in \mathcal{L}_1))\\
&\times& (1-P(s_1 \in \mathcal{L}_2|s_1 \notin \mathcal{L}_1))\\ 
&\times& P(s_1 \in \mathcal{L}_3 | s_1 \notin \mathcal{L}_1, s_1 \notin \mathcal{L}_2)\\
&=& (1-P(s_1 \in \mathcal{L}_1))\\
&\times& (1-P(s_1 \in \mathcal{L}_2|s_1 \notin \mathcal{L}_1)) \\
&\times& \sum_{l(s_2)=i=0}^{N-2} P(S_{i,2})
\times P(s_1 \in \mathcal{L}_3|{E}_{i2}=1).
\end{array}
\end{equation}
We can continue the same steps to obtain $P(Q=q)$ as in (25). The probability $P(s\in \mathcal{L}_{j+1}|{E}_{i,j}=1)$ indicates the probability that a sample $s$ with length $\ell(s)=j$ settles in the next query, when it is paired with a sample of length $i$. The probabilities in (27) is obtained by following the events described in Section~5.1 and noting the edge effect when a sample reaches a length of $N-2$ in Algorithm~3. Such a sample definitely settles in the next query once it is clear how it compares with one of the remaining two classes. 

\end{document}